\documentclass{article}

\usepackage[utf8]{inputenc}
\usepackage[margin=1in]{geometry}

\usepackage{setspace} 

\usepackage{url}
\usepackage{multicol}
\usepackage{microtype}
\usepackage{graphicx}
\usepackage{tablefootnote}
\usepackage{array}
\usepackage{subcaption}
\usepackage{booktabs} 
\usepackage{cprotect}
\usepackage{authblk}
\usepackage{wrapfig}
\usepackage{fancyvrb}
\usepackage{soul}
\usepackage{indentfirst}
\usepackage{multirow}
\usepackage{enumitem}
\usepackage{mathtools}
\usepackage{amsmath, amsthm, amssymb}
\usepackage{mathtools}
\usepackage{times}
\usepackage{diagbox}
\usepackage{xcolor} 
\usepackage{color}
\usepackage{mwe}
\usepackage{algorithmic}
\usepackage{algorithm}
\usepackage{bbm}
\usepackage[numbers,sort&compress]{natbib}
\usepackage{eso-pic} 
\usepackage{forloop}
\usepackage{url}
\usepackage{multicol}
\usepackage{microtype}
\usepackage{amsthm}
\usepackage{array}
\usepackage{subcaption}
\usepackage{booktabs} 
\usepackage{makecell}
\usepackage{indentfirst}
\usepackage{amsfonts}
\usepackage{physics}
\usepackage{mathdots}
\usepackage{yhmath}
\usepackage{cancel}
\usepackage{hyperref}
\usepackage{tabularx}
\usepackage{extarrows}
\usepackage{lastpage}
\usepackage[bottom]{footmisc}
\usepackage[usestackEOL]{stackengine}

\usepackage{tikz}
\usetikzlibrary{shapes.geometric, arrows}
\tikzstyle{rect} = [rectangle, 
minimum width=2.5cm, 
minimum height=1cm, 
text centered, 
text width=2cm, 
draw=black]

\tikzstyle{elli} = [ellipse, 
minimum width=2.5cm, 
minimum height=1cm, 
text centered, 
draw=black]

\tikzstyle{arrow} = [thick,->,>=stealth]

\newcommand{\calP}{\mathcal{P}}
\newcommand{\calW}{\mathcal{W}}
\newcommand{\E}{\mathbb{E}}
\newcommand{\R}{\mathbb{R}}

\newcommand{\bR}{\mathbb{R}}
\newcommand{\cR}{\mathcal{R}}

\newcommand{\bE}{\mathbb{E}}

\newcommand{\calB}{\mathcal{B}}

\newcommand{\W}{\mathcal{W}}
\newcommand{\calX}{\mathcal{X}}

\newcommand{\Tind}{T_{\text{point}}}
\newcommand{\QstarPoint}{Q^*_{\rm point}}
\newcommand{\QstarDist}{Q^*_{\rm dist}}

\newcommand{\rev}[1]{{\color{black}#1}}

\newcommand{\flowmodel}{\texttt{FlowDRO}}

\theoremstyle{plain}
\newtheorem{theorem}{Theorem}[section]
\newtheorem{proposition}[theorem]{Proposition}
\newtheorem{lemma}[theorem]{Lemma}
\newtheorem{corollary}[theorem]{Corollary}
\theoremstyle{definition}

\newtheorem{assumption}[theorem]{Assumption}
\theoremstyle{remark}
\newtheorem{remark}[theorem]{Remark}

\begin{document}

\title{Flow-based Distributionally Robust Optimization}
\author[1]{Chen Xu}
\author[1]{Jonghyeok Lee}
\author[2]{Xiuyuan Cheng}
\author[1]{Yao Xie\thanks{Correspondence to Yao Xie: yao.xie@isye.gatech.edu}}
\affil[1]{{\small H. Milton Stewart School of Industrial and Systems Engineering, Georgia Institute of Technology.}}
\affil[2]{{\small Department of Mathematics, Duke University}}

\date{}
\maketitle

\begin{abstract}
We present a computationally efficient framework, called \texttt{FlowDRO}, for solving flow-based distributionally robust optimization (DRO) problems with Wasserstein uncertainty sets while aiming to find continuous worst-case distribution (also called the Least Favorable Distribution, LFD) and sample from it. The requirement for LFD to be continuous is so that the algorithm can be scalable to problems with larger sample sizes and achieve better generalization capability for the induced robust algorithms. To tackle the computationally challenging infinitely dimensional optimization problem, we leverage flow-based models and continuous-time invertible transport maps between the data distribution and the target distribution and develop a Wasserstein proximal gradient flow type algorithm. In theory, we establish the equivalence of the solution by optimal transport map to the original formulation, as well as the dual form of the problem through Wasserstein calculus and Brenier theorem. In practice, we parameterize the transport maps by a sequence of neural networks progressively trained in blocks by gradient descent. We demonstrate its usage in adversarial learning, distributionally robust hypothesis testing, and a new mechanism for data-driven distribution perturbation differential privacy, where the proposed method gives strong empirical performance on high-dimensional real data.

\end{abstract}

\tableofcontents

\section{Introduction}

Distributionally Robust Optimization (DRO) is a fundamental problem in optimization, serving as a basic model for decision-making under uncertainty and in statistics for addressing general minimax problems. It aims to identify a minimax optimal solution that minimizes an expected loss over the worst-case distribution within a pre-determined set of distributions (i.e., an uncertainty set). DRO arises from various applications, including robust hypothesis testing \citep{gao2018robust,xie2021robust}, boosting \citep{blanchet2019a}, semi-supervised learning \citep{blanchet2020semi}, fair classification \citep{taskesen2020distributionally}, clustering \citep{zhu2022distributionally}, and so on; see \cite{kuhn2019wasserstein} for a more complete review. 
Inherently, DRO leads to an infinite dimensional problem, and thus, it faces a significant computational challenge in most general settings. Despite the existing efforts to solve DRO that allow analytic or approximate solutions, current approaches still have limited scalability in solving high-dimensional, large-sample problems with general risk functions. In this work, we aim to address the computational challenge using a new neural network flow-based approach; the connection with existing approaches is further discussed in Section \ref{subsec:connection-khun}.

The basic setup for DRO is given below. Let $\calX = \R^d$ be the data domain. 
Assume a real-valued \textit{risk function} $\cR(P; \phi)$ taking as inputs a $d$-dimensional distribution $P$ (with a finite second moment) and a measurable decision function $\phi \in \Phi$ in a certain function class (problem specific and possibly parametric). Assume a pre-specified scalar loss function $r: \mathcal X \times \Phi\rightarrow \mathbb R$ so that 
\begin{equation}\label{eq:risk_in_expectation}
\cR(P; \phi)=\bE_{ x \sim P} [r(x; \phi)].  
\end{equation}
Some examples of the decision function $\phi$ and loss function $r$ include $\phi$ being a multi-class classifier and $r$ being the cross-entropy loss, and $\phi$ being a scalar test function and $r$ being the logistic loss.
We are interested in solving the following minimax problem:
\begin{equation}\label{eq:DRO_general}
    \min_{\phi \in \Phi} \max_{Q \in \mathcal{B}} ~\cR(Q; \phi).
\end{equation}
In \eqref{eq:DRO_general}, $\mathcal{B}$ is a pre-defined uncertainty set that contains a set of (possibly continuous) distributions that are variations from a {\it reference distribution} $P$; this is known as the distributionally robust optimization (DRO) problem \cite{shapiro2021lectures}. In particular, we are interested in  Wasserstein DRO or WDRO (see, e.g., the original contribution \citep{mohajerin2018data}), where the $\mathcal B$ is the Wasserstein uncertainty set centered around the reference distribution induced by Wasserstein distance. WDRO receives popularity partly due to its data-driven uncertainty sets and no parametric restriction on the distributional forms considered. 

The worst-case distribution that achieves the saddle point in \eqref{eq:DRO_general} is called the Least Favorable Distribution (LFD) (also called the ``extreme distributions'' in prior works, e.g., \cite{mohajerin2018data}). In this work, we consider the problem of finding LFD for a given algorithm $\phi$, which is useful in various practical settings such as generating {\it worst scenarios} to test the algorithm and develop robust algorithms.

\subsection{Proposed: Flow-DRO}

In this paper, we propose a {\it computational} framework, a flow-based neural network called \flowmodel{} to find the worst-case distributions (LFDs) for DRO or solve the inner maximization of minimax problem \eqref{eq:DRO_general}. In particular, \flowmodel{} can efficiently compute worst-case distributions for various high-dimensional problems, thanks to the strong representation power of neural network-based generative models. The main idea is to connect the WDRO problem through Lagrangian duality to a function optimization problem with Wasserstein proximal regularization. This connection enables us to adapt the recently developed computationally efficient Wasserstein proximal gradient flow  \cite{xu2023invertible,JKOproof2023} to develop computationally efficient {\it flow-based models} parameterized by neural networks. Our framework can be viewed as a generative model for LFDs. It is thus suitable for many statistical and machine learning tasks, including adversarial learning, robust hypothesis testing, and differential privacy, leading to computationally efficient solutions and performance gain, as we demonstrated using numerical examples. 

Our main contributions are: 
\begin{itemize}

\item Develop a new Wasserstein proximal gradient descent approach to find worst-case distributions (or Least Favorable Distributions, LFDs) in WDRO by re-formulating the problem into its Wasserstein proximal form using Lagrangian duality. We introduce an alternative way to represent the LFDs through the {\it optimal transport maps} from a continuous reference measure to induce continuous LFD and use data to estimate.

\item Algorithm-wise, we adopt a new neural-network generative model approach to find LFD, called \flowmodel{}. The proposed neural network-based method can be scalable to larger sample sizes and high dimensionality, overcoming the computational challenges of previous WDRO methods. \flowmodel{} parameterize LFD by a transport map represented by a neural network, which can learn from training samples and automatically generalizes to unseen samples and efficiently generate samples from the LFDs;  we demonstrate its versatility in various applications and demonstrate the effectiveness of \flowmodel{} on multiple applications with high-dimensional problems (including images) from adversarial attack and differential privacy using numerical results.

\item Theoretically, we approach the problem in a different route, relying on the tools of optimal transport: we derive the equivalence between the original $\W_2$-proximal problem and the transport-map-search problem making use of Brenier theorem enabled by considering continuous distributions. Our theory also shows that the first-order condition of our $\W_2$-proximal problem using Wasserstein calculus leads to an optimality condition of solving the Moreau envelope without assuming the convexity of the objective. As a by-product, we recover the closed-form expression of the dual function involving the Moreau envelope of the (negative) loss, consistent with existing work, and highlight the computational advantages of using our alternative optimal transport map reformulation.
\end{itemize}

To the best of our knowledge, \flowmodel{} is the first work that finds the worst-case distributions in DRO using flow-based models. However, we would like to emphasize that our approach is general and does not rely on neural networks; one can potentially use an alternative representation of the transport maps (e.g., \cite{hutter2021minimax}) in low-dimensional and small sample settings for stronger learning guarantees. 
\textcolor{black}{In the context of minimizing an objective functional in probability space, \citep{kent2021modified} proposed an infinite-dimensional Frank-Wolfe procedure.
The work leveraged the strong duality result in DRO \citep{blanchet2019quantifying} 
(see more in Section \ref{subsec:dual-form-W-prox}, Eqn. \eqref{eq:dual-LFD-G-2}) 
to compute the Wasserstein gradient descent steps. 
Our work focuses on the sub-problem of finding LFD in DRO, and our algorithm uses neural networks to tackle distributions in high dimensional space.}

\subsection{Motivating example: Why continuous density for LFD?}

One may quickly realize that finding LFD is an infinite-dimensional optimization problem that is particularly challenging in high dimensions and general risk functions.
A useful observation that occurs in such infinite-dimensional optimization problem is that the worst-case distribution solution of the WDRO problem \eqref{eq:DRO_general} turns out to be discrete, as shown in the original paper \citep{mohajerin2018data} and various follow-up works including \cite{xie2021robust} for the distributionally robust hypothesis test. 
This particular solution structure does help to overcome the computational challenge caused by the infinite-dimensional optimization problem. 

However, the discrete nature of LFD, as coming from the WDRO formulation, is not desirable in practice, as explained in the following. First, there is a significant computational challenge. The method is not scalable to large datasets: the discrete WDRO formulation will require solving a Linear Program (LP) with the number of decision variables to be $\mathcal O(n^2)$, where $n$ is the total number of training data points and the complexity of solving an LP is typically quadratic on the number of the decision variable. Such computational complexity for problems with thousands of training data points can be prohibitive (e.g., the MNIST handwritten digit example in our later section uses $\sim$5000 samples per class). So typically, the current WDRO formulation usually can only be used to find discrete LFDs for small sample settings (e.g., \cite{zhu2022distributionally,xie2021robust}). 
Second, the discrete LFD will limit the {\it generalization} capability of the resulting algorithm. In machine learning applications, when we develop a robust detector (binary classifier) using DRO \cite{gao2018robust,xie2021robust}, the LFD is discrete with a support on the training data set, as shown in Fig. \ref{fig:WDRO_vs_FlowDRO}(a). As a result, the optimal detector is also {\it only defined} on the support of training data points. Such an optimal detector does not generalize in that,  given a new test sample, we cannot directly apply it to the test data if it does not coincide completely with one of the training data points. An ad-hoc approach could be to ``interpolate'' the optimal detector on the training samples by convolving with a smoothing kernel (such as a Gaussian kernel); however, this will lose the property of the original minimax optimality of the detector. It would be better to seek continuous worst-case distributions (LFDs) when we solve the minimax problem. Thus, we may want to add a constraint in the formulation and consider the uncertainty set as the intersection of the Wasserstein uncertainty set and the set of continuous functions.

Suppose we would like to find {\it continuous worst-case distribution} instead for the above consideration. However, if one restricts $\mathcal P$ in the minimax problem \eqref{eq:DRO_general} to be the Wasserstein uncertainty set {\it intersecting all continuous distribution functions}, that will lead to an even more difficult infinite dimensional problem involving distribution functions, and the (discrete) solution structure property no longer holds. This brings out the main motivation of our paper: we will introduce {\it a neural network (NN) approach to solve minimax problem} leveraging the strong approximation power of NN and that they {\it implicitly regularize} the solution to achieve continuous density. To carry out the plan, we need a carefully designed NN architecture and training scheme leveraging the recent advances in {\it normalizing flow} to represent distribution functions. Recently, there have also been works considering entropy regularized Wasserstein uncertainty sets, called the Sinkhorn DRO problems  \citep{wang2021sinkhorn}, which lead to continuous LFDs with kernel-type solutions. Still, it is more suitable for low-dimensional problems due to the nature of the kernel solutions.

\begin{figure}[!t]
    \begin{minipage}{0.49\textwidth}
        \includegraphics[width=\linewidth]{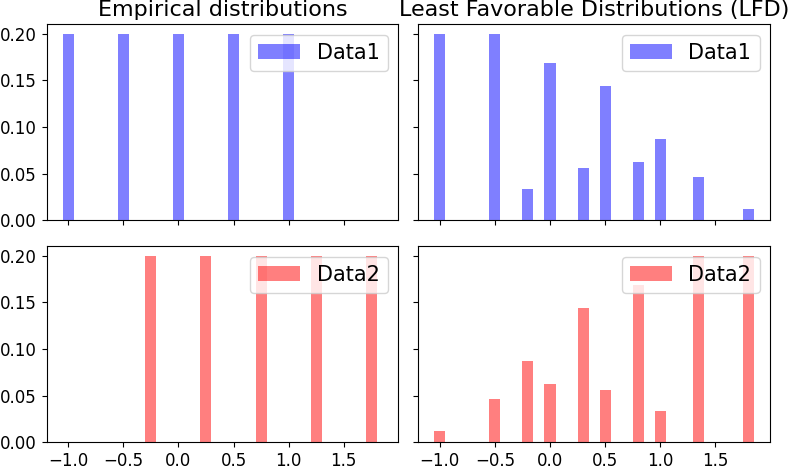}
        \subcaption{WDRO on training samples}
    \end{minipage}
    \begin{minipage}{0.49\textwidth}
        \includegraphics[width=\linewidth]{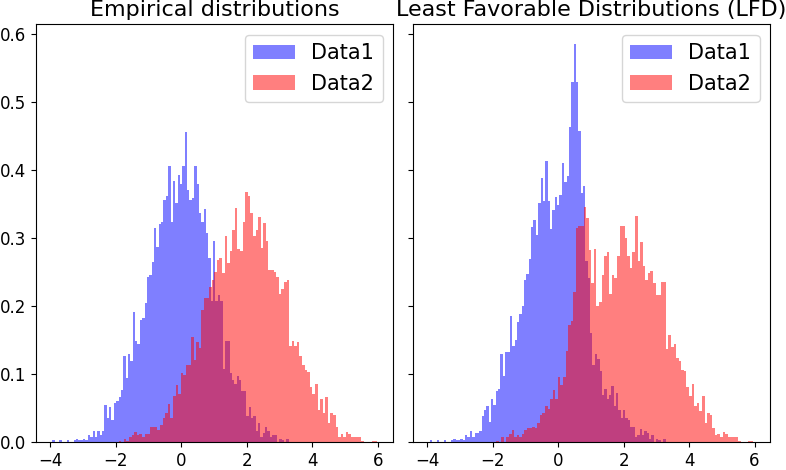}
        \subcaption{Proposed \flowmodel{} on test samples}
    \end{minipage}
    \caption{Comparison of WDRO and \flowmodel{} on the 1D example following \citep[Figure 1]{xie2021robust}. 
    {\textbf{Left of (a) and (b):}} Empirical distributions of two sets of \textit{training} (shown in (a)) and \textit{test} (shown in (b)) samples from $\mathcal N(0,1)$ and $\mathcal N(2,1.2)$. 
    {\textbf{Right of (a) and (b):}} Least-favorable distributions (LFD) found by WDRO and \flowmodel{}, where LFDs are within the $W_2$ ball \eqref{eq:W2_uncertainty_set} with radius $\varepsilon=0.1$. As expected, the LFDs overlap more with each other than the empirical distributions do.
    Note that WDRO solves a convex problem to obtain the LFD by moving the probability mass on \textit{discrete training samples}. In particular, WDRO is not generalized to a test sample unless it coincides exactly with some of the training samples. In comparison, \flowmodel{} yields a one-to-one continuous-time transport map that can be directly applied to training and test samples. The resulting LFD is also continuous, as it is the push-forward distribution by the transport map on the underlying continuous data distribution.}
    \label{fig:WDRO_vs_FlowDRO}
\end{figure}

\subsection{Flow-based generative models}

Recently, diffusion models and closely related flow-based models have drawn much research attention, given their state-of-the-art performance in image generation; see \cite{JKOproof2023} for a complete summary. Flow-based generative models enjoy certain advantages in computing the data generation and the likelihood and have recently shown competitive empirical performance. They can be understood as continuous-time models that gradually transform the input distribution $P$ into a target distribution $Q$. These models are popularized in the context of normalizing flow, where the target distribution $Q=\mathcal{N}(0, I_d)$, the standard multivariate Gaussian \citep{nflow_review}. They can be largely categorized into discrete-time flows \citep{iResnet,ResFlow,xu2022IGNN} and continuous-time flows \citep{FFJORD,onken2021ot,xu2023invertible,xu2023q}. The discrete-time flows can be viewed as Euler approximation of the underlying continuous-time probability trajectory, where the continuous-time flows are based on neural ordinary differential equation (NeuralODE) \citep{chen2018neural} to learn the probability trajectory directly. 

We remark that, unlike normalizing flow and flow between arbitrary pre-specified pairs of distributions, our flow model tries to learn the worst-case distribution $Q^*$ that maximizes certain risk functions.
Compared with other flow-based generative models,  such as the traditional settings of normalizing flow or pre-specified target distributions, \flowmodel{} does not choose a target distribution {\it a-priori}, which is learned by maximizing the objective function. 

We also note that different from training generative adversarial networks (GAN) \citep{goodfellow2014generative} that may also generate worst-case samples, our flow-based approach can be more stable during training as it involves neither auxiliary discriminators nor inner loops. Compared with recent works on flow-based generative models\cite{xu2023invertible,JKOproof2023}, where only KL divergence was considered for the loss function, we consider general loss as motivated by various applications. 

\subsection{Applications}

\flowmodel{} can also directly benefit several applications, which can be formulated as DRO problems, as we present in more detail in section \ref{sec:applications}.
First, in the case of an adversarial attack, our flow model is an \textit{attacker} that can find the distribution causing the most disruption to existing systems. This is especially important for engineering system design. For example, in power systems, we are interested in understanding the resiliency of a power network. Given limited historical observations, we are interested in discovering whether any unseen scenario may cause a catastrophic consequence to the system. Finding such scenarios can help evaluate engineering systems and improve network resiliency.
Second, in the case of differential privacy (DP), our flow model acts as a \textit{distributional perturbation mechanism} to dataset queries. Upon finding the worst-case distribution around the data distribution over queries, we can provide much protection against potential data disclosure and/or privacy loss. This is extremely useful in high-stakes settings where sensitive information must be protected.
We also note that the existing DP framework is largely not data-driven. Specifically, DP mechanisms often take the simple approach of adding i.i.d. noise to each dimension of queries, and the noises are unrelated to data. There is growing interest in developing data-driven mechanisms by exploiting the query structure or the data distribution, which may bring potential performance gains. However, finding such optimal perturbation subject to the privacy constraint poses a computational challenge, which we try to address through the proposed \flowmodel{}.

The rest of the paper is organized as follows. Section \ref{sec:framework} formally introduces the framework of solving for the worst-case distribution, along with theoretical analyses. Section \ref{sec:method} describes the \flowmodel{} method and the concrete training algorithm. Section \ref{sec:applications} considers several important applications for which \flowmodel{} can be used. Section \ref{sec:experiments} shows numerical results of \flowmodel{} on high-dimensional problems. Section \ref{sec:discussion} concludes the work with discussions. All proofs are delegated to the appendix.

\section{Framework}\label{sec:framework}

Below, we focus on Wasserstein-2 ($\W_2$) in this work, and extensions to $\W_p$ with other $p$ are left to future studies.
Let $\calX = \R^d$, and denote by $\calP_2(\calX )$ the space of all distributions on domain $\calX$ that have a finite second moment,
that is, $\calP_2(\calX) := \{ P, \, \int_\calX |x|^2 dP(x) < \infty \}$. 
Define $\calP_2^r(\calX):= \{ P \in \calP_2, \, P \ll {\rm Leb} \}$, that is, all distributions in  $\calP_2(\calX)$ that also have continuous densities (absolute continuous with respect to the Lebesgue measure). We may omit $(\calX)$ in the notation $\calP_2$ and $\calP_2^r$.

\subsection{Dual formulation and Wasserstein proximal problem}\label{subsec:dual-form-W-prox}

The $\W_2$-distance between two distributions in $\calP_2$ is defined by 
\begin{equation}\label{eq:deff-W2-Kantorovich}
    \W_2^2( \mu, \nu ) := \inf_{\pi \in \Pi ( \mu, \nu)}
        \int_{ \R^d \times \R^d} \| x-y \|^2 d\pi(x,y),
\end{equation}
where $\Pi (\mu, \nu)$ denotes the family of all joint distributions with $\mu$ and $\nu$ as marginal distributions, called the couplings of $\mu$ and $\nu$. For any given $\nu \in \calP_2$, the functional $\W_2^2( \cdot, \nu)$ maps from $\calP_2$ to $[0, \infty)$, by the following lemma.
\begin{lemma}\label{lemma:finite-W2}
For any $\mu, \nu \in \calP_2$, $\W_2( \mu, \nu) < \infty $.
\end{lemma}

Let $\calB_\varepsilon(P)$ be the $\W_2$-ball  in $\calP_2 $ around the reference distribution $P$ of radius $\varepsilon >0$, namely
\begin{equation}\label{eq:W2_uncertainty_set}
\calB_\varepsilon (P) = \{ Q \in \calP_2, \, \W_2(Q, P)  \le \varepsilon \}. 
\end{equation}
As explained in the introduction, we will focus on the case where $P$ has (continuous) density, that is $P \in \calP_2^r$.
We focus on the inner loop (the ``max'') of the min-max problem \eqref{eq:DRO_general} where the uncertainty set $\calB = \calB_\varepsilon(P)$. 
For fixed decision function $\phi$,
we cast the maximization as a minimization by defining $V(x): = - r(x ; \phi)$.
The central problem we aim to solve in the paper is to find the LFD, which can be equivalently written as the following:
\begin{equation}\label{eq:problem-W2-trust-region}
     \min_{Q \in \calB_\varepsilon (P) } ~ \bE_{x \sim Q} V(x), \quad \texttt{\{LFD problem\}}.
\end{equation}
The idea is to convert the uncertainty set constraint as a regularization term of the original objective function by introducing a Lagrangian multiplier. Then, we can leverage this connection to build a Wasserstein gradient flow type of algorithm to solve the LFD problem.

\paragraph{Dual form and proximal problem.}
The constrained minimization \eqref{eq:problem-W2-trust-region} is a trust region problem.
It is well known that in vector space, 
trust-region problem 
can be solved by a proximal problem where the Lagrangian multiplier defined through $\lambda >0$ corresponds to the radius $\varepsilon$ \cite{parikh2014proximal}.
Specifically, consider the {\it dual form} of the LFD problem \eqref{eq:problem-W2-trust-region}, which can be written as
\begin{equation}\label{eq:dual-LFD-G}
\sup_{\lambda \ge 0}~
G( \lambda ) : = 
\min_{Q \in \calP_2} \E_{x \sim Q} V(x) + \lambda ( \W_2^2( P, Q) -\varepsilon^2),
\quad \texttt{\{dual form\}}.
\end{equation}
We restrict ourselves to the case when $\lambda > 0$, and introduce the change of variable 
\[
\lambda = \frac{1}{2\gamma }, \quad \gamma > 0.
\]
After dropping the constant term $\lambda \varepsilon^2$ in \eqref{eq:dual-LFD-G}, we obtain the following Wasserstein {\it proximal problem}
\begin{equation}\label{eq:problem-W2-proximal-GD}
    \min_{ Q \in \calP_2(\calX ) }  
    \E_{ x \sim Q} V(x) + \frac{1}{2\gamma } \calW_2^2(P, Q),  \quad \texttt{\{proximal problem\}}.
\end{equation}
The $\W_2$-proximal problem can be viewed as the Moreau envelope (or the Moreau-Yosida regularization) in the Wasserstein space \cite{moreau1965proximite}.
Similar to the vector-space case, we will have a correspondence between \eqref{eq:problem-W2-trust-region} and \eqref{eq:problem-W2-proximal-GD}, see Remark \ref{rk:TR-PR-correspondence}, which will be introduced in Section \ref{sec:theory} after we derive the first-order optimality conditions of the two problems.

\paragraph{Explicit form of dual function.}
It has been pointed out in several prior works that the dual form can be reformulated using the Moreau envelope of the (negated) loss function under different scenarios  
\cite{blanchet2019quantifying,gao2023distributionally,zhang2022simple}.
Specifically, the explicit expression of the dual form \eqref{eq:dual-LFD-G} is written as
\begin{equation}\label{eq:dual-LFD-G-2} 
\sup_{\lambda \ge 0}~
G( \lambda )  := 
 \E_{x \sim P} \inf_z \left[ V(z) + \lambda \| z - x\|^2   \right]
- \lambda\varepsilon^2.
\end{equation}
Assuming $\lambda > 0$, the dual function $G$ in \eqref{eq:dual-LFD-G-2}  can be  equivalently written as
\begin{equation}\label{eq:dual-LFD-g-gamma} 
G\left(\frac{1}{2\gamma} \right)  = 
 \E_{x \sim P}~ u( x, \gamma)
- \frac{\varepsilon^2}{2\gamma},
\end{equation}
where $u(x, \gamma)$ is the Moreau envelope of $V$ defined as 
\begin{equation}\label{eq:def-Moreau-V}
u(x, t) :=
\inf_z \left[ V(z) + \frac{1}{2 t} \| z - x\|^2   \right], \quad t > 0.    
\end{equation}
This form of the dual function echos the observation that the Wasserstein proximal operator for the functional in the form of $\varphi( \mu) = \int V d\mu $ can be solved by the proximal operator (Moreau envelope) of $V$, as has been pointed out in the PDE literature, see e.g. \cite{BOWLES201530}.

We will recover the same explicit form of the dual function under certain conditions in Section \ref{sec:theory} where the Moreau envelope has unique minimizer $z$ for each $x$, see Corollary \ref{cor:dual-form}. 
Meanwhile, from the computational perspective, the Moreau envelope $u(x, \gamma)$ may still be challenging to solve %
in high dimensions, among other algorithmic challenges. We further discuss this and the connections to previous studies of the dual form in Section \ref{subsec:connection-khun}. 
Instead of using the dual form \eqref{eq:dual-LFD-G-2}, we propose to solve the dual problem (equivalently the $\W_2$-proximal problem \eqref{eq:problem-W2-proximal-GD}) by parameterizing a transport map $T: \R^d \to \R^d$ possibly by a neural network, to be detailed in the next section.

\subsection{Solving the Wasserstein proximal problem by  transport map}

We show that the problem \eqref{eq:problem-W2-proximal-GD} that minimizes over $Q$ 
can be solved by minimizing over the transport map $T$, which will pushforward $P$ to $Q$.
(Recall that for $T:\calX \to \calX$, 
the {\it pushforward} of a distribution $P$ is denoted as $T_\# P$,  such that 
$T_\# P(A) = P( T^{-1}(A))$ for any measurable set $A$.) This reformulation is rooted in the Monge formulation of the Wasserstein distance.

When $P \in \calP_2^r$, the Brenier theorem allows a well-defined and unique optimal transport (OT) map from $P$ to any $\mu \in \calP_2$.
For completeness, we include the argument as follows. We denote by $T_P^{\square}$ the OT map from $P$ to ${\square} \in \calP_2$, which is defined $P$-a.e.,
and $( T_P^\mu )_\# P = \mu$.
Given any $\mu \in \calP_2$, for any $T:\R^d \to \R^d$ s.t. $T_\# P = \mu$,
$({\rm I_d}, T)_\# P$ is a coupling of $P$ and $\mu$, and thus
\begin{equation}\label{eq:W2-less-than-cost}
    \W_2^2( \mu, P)
    \le 
    \E_{ x \sim P} \| x - T(x)\|^2.
\end{equation}
The problem of minimizing the r.h.s. of \eqref{eq:W2-less-than-cost} over all $T$ that pushforwards $P$ to $\mu$ is known as the Monge Problem. 
{\it By Brenier Theorem, when $P \in \calP_2^r$,
the OT map attains the minimum of the Monge problem}, that is,  
\begin{equation}
    \W_2^2( \mu, P) 
    = \E_{ x \sim P} \| x - T_P^\mu(x)\|^2.
\end{equation}

We introduce the following transport map minimization problem corresponding to the $W_2$-proximal problem \eqref{eq:problem-W2-proximal-GD}
\begin{equation}\label{eq:problem-W2-T}
    \min_{ T: \calX \to \calX, \, T_\# P \in \calP_2(\calX) }  
    \E_{ x \sim P} \left( V \circ T(x) + \frac{1}{2\gamma }  \| x - T(x)\|^2 \right).
\end{equation}
The formal statement of the equivalence between \eqref{eq:problem-W2-proximal-GD} and \eqref{eq:problem-W2-T} is by applying Proposition \ref{prop:W2-PR-by-T} with $\varphi( \mu ) := \E_{ x \sim \mu } V(x)$, which is assumed to be finite for any $\mu \in \calP_2(\calX)$, and $\lambda =1/2\gamma >0$.
The proof follows a similar argument as in \cite[Lemma A.1]{xu2023invertible} and is included in Appendix \ref{app:proof} for completeness.

\begin{proposition}[Equivalent solution by transport map]
\label{prop:W2-PR-by-T}
Suppose $\varphi: \calP_2(\calX ) \to (-\infty, \infty)$,
$P \in \calP_2^r(\calX)$,
and define $L^2(P):=\{v: \R^d \to \R^d, \, \E_{x \sim P} \| v(x) \|^2 < \infty  \}$.
 For any $\lambda > 0$, the following two problems
\begin{equation}\label{eq:def-PR-mu-lemma}
\min_{\mu  \in \calP_2 (\calX) } 
L_\mu ( \mu ) = 
\varphi( \mu)  + \lambda \W_2^2( P,\mu),
\end{equation}
\begin{equation}\label{eq:def-PR-T-lemma}
\min_{T \in L^2(P)}
L_T ( T )=  \varphi( T_\# P)  + \lambda \E_{x \sim P}  \| x-T(x) \|^{2},
\end{equation}
satisfy that

(a) If $T^*$ is a minimizer of \eqref{eq:def-PR-T-lemma}, %
then $ (T^*)_\# P$ is a minimizer of \eqref{eq:def-PR-mu-lemma}.

(b) If $\mu^*$ is a minimizer of \eqref{eq:def-PR-mu-lemma}, %
then the OT map from $P$ to $\mu^*$  minimizes \eqref{eq:def-PR-T-lemma}.

In both cases, the minimum $L_\mu^*$ of \eqref{eq:def-PR-mu-lemma}
and the minimum $ L_T^*$ of \eqref{eq:def-PR-T-lemma}
equal. 
\end{proposition}

We will solve \eqref{eq:problem-W2-T} by parameterizing the transport map $T$ by a flow network on $[0, \gamma]$ and learn $T$ by setting \eqref{eq:problem-W2-T} as the training objective. Details will be introduced in section \ref{sec:method}. %

\subsection{Connection to existing Wasserstein DRO}\label{subsec:connection-khun}

The dual form \eqref{eq:dual-LFD-G-2} has been derived in several works under different settings \cite{mohajerin2018data,kuhn2019wasserstein,blanchet2019quantifying,gao2023distributionally,zhang2022simple} - noting that we define $V$ to be the negative loss, thus \eqref{eq:dual-LFD-G-2} is ``sup-inf'', while the dual of the original LFD problem is ``inf-sup''. Below, we discuss the connection under our framework.

\subsubsection{Reduction in the case of discrete reference measure}

We show a connection of our problem to the known result in the literature (see, e.g., \cite{kuhn2019wasserstein}): 
when the reference distribution $P$ is discrete (rather than having a density, i.e., a continuous distribution considered in our setting), 
\cite{kuhn2019wasserstein} proved a  ``strong duality'' result \eqref{eq:khun-thm7}.
Here, we show that the dual form \eqref{eq:dual-LFD-G} will end up being in the same as the dual form therein (which is equivalent to \eqref{eq:dual-LFD-G-2}), and the argument is via \eqref{eq:problem-W2-T} which illustrates the role played by the transport map $T$.
 This is an interesting connection because the dual form in \cite[Theorem 7]{kuhn2019wasserstein} plays a role in reducing the original complex infinite-dimensional problem to a finite-dimensional problem to solve the discrete LFD
\cite{mohajerin2018data,kuhn2019wasserstein}.
However, such reduction only happens when the center of the uncertainty set $P$ is discrete; when $P$ is not discrete rather than continuous, the case considered in our paper, we need to develop an alternative computational scheme.

When $P$ is an empirical distribution (thus discrete), we denote $P = \hat P$ and 
\[
\hat{P} = \frac{1}{n}\sum_{i=1}^n \delta_{x_i}, 
\]
for a dataset $\{ x_i \}_{i=1}^n$. 
We first restate \cite[Theorem 7]{kuhn2019wasserstein} using our notations ($p=2$ in $\W_p$): 
\begin{equation}\label{eq:khun-thm7} 
\sup_{ Q \in \calB_\varepsilon (\hat P)} 
\E_{x \sim Q} r(x, \phi)
= 
\inf_{\lambda \ge 0} \left\{
 \E_{x \sim \hat P} \sup_z \left[ r(z; \phi) - \lambda \| z - x\|^2   \right]
+ \lambda\varepsilon^2 \right\}.
\end{equation}
Note that the dual form (the r.h.s. of \eqref{eq:khun-thm7})
 is equivalent to \eqref{eq:dual-LFD-G-2} replacing $P$ to be $\hat P$ (and swapping to ``sup-inf'').

Recall the dual form \eqref{eq:dual-LFD-G} where we take $P = \hat P$. After dropping the constant term $\lambda \varepsilon^2$, 
the following proposition gives the explicit expression of the dual function $G(\lambda )$.
We believe similar arguments have appeared in the literature, and we include proof for completeness.

\begin{proposition}[Dual form for discrete $P$]\label{prop:dual-discrete}
Given $\lambda > 0$, suppose $\forall i=1,\dots, n$,  $\inf_{z}  \left[ V(z) + \lambda \| x_i - z\|^2 \right]$ attains its minimum at some point $z_i \in \R^d$, then 
\begin{equation}\label{eq:dual-form-hatP}
\min_{Q \in \calP_2} \E_{x \sim Q} V(x) + \lambda  \W_2^2( \hat P, Q)
= \E_{x \sim \hat P} \inf_{z}  \left[ V(z) + \lambda \| x - z\|^2 \right].
\end{equation}    
\end{proposition}

As a result, we have
\[
G(\lambda) = 
	\E_{x \sim \hat P} \inf_{z}  \left[ V(z) + \lambda \| x - z\|^2 \right]-   \lambda \varepsilon^2.
\]
This dual function is equivalent to the dual form on the r.h.s. of \eqref{eq:khun-thm7}, recall that $V(x) = -r (x;\phi)$.

It will be illustrative to derive the r.h.s. of \eqref{eq:dual-form-hatP} formally from the transport-map-search problem \eqref{eq:problem-W2-T}: with $P = \hat P$, we obtain
\begin{equation}
\min_{ T: \R^d  \to \R^d }  
    \frac{1}{n}
    \sum_{i=1}^n
    \left( V \circ T( x_i ) +  \lambda   \| x_i - T(x_i)\|^2 \right).
\end{equation}
Since $x_i$ are discrete points, the effective variable are $z_i: = T(x_i)$, that is, the minimization is equivalent to
\[
\min_{ \{ z_i \}_{i=1}^n, \, z_i \in \R^d }  
    \frac{1}{n}
    \sum_{i=1}^n
    \left( V(z_i) + \lambda \| x_i - z_i\|^2 \right).
\]
This minimization is decoupled for the $n$ points $z_i $, and the minimization of each $z_i$ 
This gives that 
\[
\min_{T: \R^d \to \R^d} 
	\E_{x \sim \hat P} \left[ V( T(x) ) + \lambda  \| x - T(x)\|^2 \right]
= \E_{x \sim \hat P} \inf_{z}  \left[ V(z) + \lambda \| x - z\|^2 \right].
\]

\subsubsection{Connection to the dual formulation of WDRO}\label{subsec:connection-zhang}

Prior works have also attempted to use the dual formulation to 
evaluate the {\it objective value} under the worst-case distribution $L:= \max_{Q \in \mathcal{B}} ~\cR(Q; \phi)$. 
For example, the strong duality was obtained in a general setting in \cite{blanchet2019quantifying}.
This approach is helpful to evaluate the objective function value under the worst-case distribution directly and, thus, can help to develop a robust algorithm $\phi(\theta)$ with respect to its parameter $\theta$. 
However, the approaches along this line of thought may encounter certain limitations in practice.
First, it is well understood that general functions do not admit explicit formulas for their proximal operators, that is, finding the Moreau envelop, namely finding the inner pointwise supermum problem $\sup_{z \in \mathcal X}
\left\{ r(z;\phi) - \lambda \|z-x \|^2
\right\}$ 
does not have a closed-form solution.
In cases where this inner-loop optimization is convex and differentiable, one can use standard iterative solvers to find the supremum $z$, yet this calls for a solution of $z$ for each $x$ point-wisely. 
When the objective $r$  is non-linear and non-convex, there can be other algorithmic complications; see a recent discussion in \cite{osher2023hamilton}.
In addition, computational challenges arise in evaluating the expectation $\E_{x \sim P}  $ in \eqref{eq:dual-LFD-G-2}.
When the reference distribution $P$ is not discrete, the expectation may not have a closed-form expression, or one may have to rely on sampling from $P$ and perform a Sample Average Approximation (SAA), and the accuracy of SAA in high dimension relies on processing a large number of data samples. 
At last, even the formulation \eqref{eq:dual-LFD-G-2} can be useful for finding robust algorithms that minimize the worst-case loss, %
the LFD cannot be identified using the formulation, and one cannot sample from the LFD, which is desirable for applications such as adversarial scenario generation.

Theoretically, our analysis in this work obtains the dual form in a different route, primarily relying on the theoretical tools of optimal transport. 
We will 
derive the dual form \eqref{eq:dual-LFD-G-2}  in Section \ref{sec:theory} 
by showing that the first-order condition of the $\W_2$-proximal problem in Wasserstein calculus leads to an optimality condition of solving the Moreau envelope (Corollary \ref{cor:dual-form});
To justify the algorithm based on parameterizing the transport map, we derive the equivalence between the distribution-search problem 
(the original $\W_2$-proximal problem) 
and the transport-map-search problem in Proposition \ref{prop:W2-PR-by-T} 
making use of the Brenier theorem. 
These theoretical analyses utilize that the LFD has a continuous density. 

\section{Theory}\label{sec:theory}

In this section, we derive the first-order optimality condition for the \texttt{LFD problem} \eqref{eq:problem-W2-trust-region} and the \texttt{proximal problem} \eqref{eq:problem-W2-proximal-GD}, when considering the primal formulation to find LFD. Although the derivation is elementary, such characterization seems to not exist in the literature as far as we know, and the characterization may shed some insights into the Wasserstein space nature of the problem. Moreover, the first-order conditions also help establish the dual form of the LFD problem.

\subsection{Preliminaries}

\paragraph{Notations.}
To state the main result, we first introduce some necessary notations. For a distribution $P$ on $\R^d$, define the second moment 
$M_2(P) := \int_{\R^d} \|x\|^2 dP(x)$.
Given $\mu \in \calP_2$,
the $L^2$ space denoted by $L^2 (\mu )$ is for the vector fields $v: \R^d \to \R^d$.
For $u, v : \R^d \to \R^d$, the inner-product
\[
\langle u , v\rangle_\mu := \int_{\R^d} u(x)^T v(x) d\mu(x),
\]
and the $L^2$-norm is defined as 
\[
\| u \|_\mu^2 = \int_{\R^d} \| u(x) \|^2 d\mu(x).
\]

We will use $v \in L^2(\mu)$ as a (small) displacement field; that is, we will consider the perturbation of $\mu$ to $( {\rm I_d} +v)_\# \mu $.
By Lemma \ref{lemma:L2-perturb}, if $v \in L_2(\mu)$, then  $({\rm I_d} + v)_\# \mu$ remains in  $\calP_2$.
 
\begin{lemma}\label{lemma:L2-perturb}
    If $\mu \in \calP_2$, $ T \in L^2(\mu)$,
    then $ T_\# \mu \in \calP_2$.
\end{lemma}

We introduce notations of the following key functionals on $\calP_2$,
\begin{equation}\label{eq:def-phi-psi}
\varphi( \mu):= \int_{\R^d} V(x) d\mu(x),
\quad 
\psi( \mu) := \frac{1}{2}\W_2^2( \mu , P). 
\end{equation}
Then the \texttt{LFD problem} can be written as 
\[
\min_{Q \in \calP_2, \, \psi(Q) \le \varepsilon^2/2}
\varphi (Q).\]    
Because $\calP_2$ lies inside the manifold of all distributions over $\R^d$, the notion of calculus and convexity of $\varphi$ and $\psi$ in $\calP_2$ are very different from the case in vector space.
However, it is reasonable to expect certain optimization results in vector space to find the analog here.
The analysis here centers around the (sub)differential of $\varphi$ and $\psi$ in $\calP_2$, which has been systematically studied in the analysis literature, see Sections 9 and 10 of \cite{ambrosio2005gradient}.
Our argument follows the constructions in \cite{ambrosio2005gradient}, simplifying the notions and making the theoretical argument self-contained.

\subsection{$\W_2$-differentials}

Recall that $\varphi$ defined in \eqref{eq:def-phi-psi} is a linear function of $\mu$;
however, being linear generally does not imply that the functional is ``convex'' on $\calP_2$.
Specifically, the convexity in $\calP_2$ needs to be defined along geodesics (or general geodesics).
As a simple example,
$\mu_0 = \delta_{x_0}$ and $\mu_1 =\delta_{x_1}$,
then the geodesic from $\mu_0$ to $\mu_1$ in $\calP_2(\R^d)$ will consists the Dirac measure $\mu_t = \delta_{x_t}$, $t \in [0,1]$
where $x_t$ lies on the geodesic from $x_0$ to $x_1$ in $\R^d$ namely the line connecting the two points. 
For any $t \in [0,1]$, $\varphi( \mu_t ) = V( x_t)$.
Then, unless the function $V$ is convex, the functional $\varphi(\mu)$ will not be convex along the geodesic from $\mu_0$ to $\mu_1$.

We first introduce a lemma concerning the behavior of $\varphi$ when the distribution is perturbed in $\calP_2$. For $\varphi(\mu) = \int V d\mu$, we introduce the following assumption on the potential $V$ (without assuming its convexity).

\begin{assumption}[$L$-smooth loss]\label{assump:V}
$V$ is $L$-smooth on $\R^d$ for some $L > 0$, meaning that $V$ is $C^1$ on $\R^d$ and $\nabla V$ is $L$-Lipschitz.
\end{assumption}
\begin{lemma}[Strong differential of $\varphi$]\label{lemma:diff-phi}
Under Assumption \ref{assump:V}, $\varphi: \calP_2 \to (-\infty, \infty)$.
At any $\mu \in \calP_2$, $\nabla V \in L^2(\mu)$ and
$\varphi$ has strong $\W_2$-differential
\[
\nabla_{\W_2} \varphi( \mu ) = \nabla V, \quad \mu\text{-a.e.},
\]
in the sense that $\forall v \in L^2(\mu)$, $\| v \|_\mu = 1$, and $\delta \to 0+$,
\begin{equation}\label{eq:strong-diff-phi}
    \varphi( ( {\rm I_d} + \delta v)_\# \mu) = \varphi( \mu) 
    + \delta \langle \nabla V, v \rangle_\mu + o (  \delta ).
\end{equation}
\end{lemma}

For $\psi(\mu) = \frac{1}{2} \W_2^2( \mu, P)$, where $P \in \calP_2^r$ is fixed,
the $\calP_2$ calculus is more conveniently derived in a neighborhood of $\mu \in \calP_2^r$.
It is known that the $\W_2$ differential (both sub- and super-differential) of $\psi$ at $\mu \in \calP_2^r$ has the expression as  $({\rm I_d} - T_\mu^P)$, see e.g.  \cite[Corollary 10.2.7]{ambrosio2005gradient} where the subdifferential is defined not in the ``strong'' sense.
Here, we give a lemma on the strong super-differential of $\psi$
(i.e. strong subdifferential of $-\psi$),
which suffices for our purpose.

\begin{lemma}[Strong super-differential of $\psi$]\label{lemma:superdiff-psi}
Let $P \in \calP_2$ be fixed, 
for any $\mu \in \calP_2^r$,  the optimal transport map $T_\mu^P $ is defined $\mu$-a.e.,
and the functional $-\psi$ has strong $\W_2$-subdifferential at $\mu$, 
\[
- ( {\rm I_d} - T_\mu^P ) \in \partial_{\W_2} ( - \psi )( \mu ), 
\]
in the sense that $\forall v \in L^2(\mu)$, $\| v \|_\mu = 1$, and $\delta \to 0+$,
\begin{equation}\label{eq:strong-superdiff-psi}
    \psi( ( {\rm I_d} + \delta v)_\# \mu) 
    \le
     \psi( \mu) 
    + \delta \langle  {\rm I_d} - T_\mu^P , v \rangle_\mu + o (  \delta ).
\end{equation}    
\end{lemma}
One remark is that, in the above lemma, we only need $P\in\calP_2$ and no need to have density.
The unique existence of the optimal transport map $T_\mu^P$ needs $\mu$ to have density.

\subsection{First-order condition of \texttt{LFD problem}}

We will analyze the first-order condition around a local minimum of the \texttt{LFD problem}
based on the relations \eqref{eq:strong-diff-phi} and \eqref{eq:strong-superdiff-psi}.
While \eqref{eq:strong-diff-phi} holds at any $\mu \in \calP_2$,
 \eqref{eq:strong-superdiff-psi} requires $\mu \in \calP_2^r$, 
 Thus, we assume the minimizer $Q$ of the TR problem has density.
\begin{assumption}[Minimizer of \texttt{LFD problem} in $\calP_2^r$]\label{assump:minimizer-in-P2r}
  The problem \eqref{eq:problem-W2-trust-region} attains a (local) minimum at $Q \in \calP_2^r$.    
\end{assumption}

\begin{remark}
 In our theory, we do not use the assumption $P \in \calP_2^r$ explicitly, 
however, if $P$ does not have density, then usually the minimizer $Q$ will not have density,
e.g., in the discrete LFD considered in \cite{mohajerin2018data,kuhn2019wasserstein}. 
Thus, we assume $P$ has a density so that Assumption \ref{assump:minimizer-in-P2r} can be reasonable.
\end{remark}

\begin{theorem}[First-order condition of \texttt{LFD problem}]\label{prop:TR}
Let $P \in \calP_2$ be fixed, 
under Assumptions \ref{assump:V} and \ref{assump:minimizer-in-P2r}, 
at a local minimizer $Q$ of \eqref{eq:problem-W2-trust-region} which is in $\calP_2^r$, 
\begin{itemize}
\item[(i)] $\calB_\varepsilon$ constraint not tight:
 If $\W_2(Q, P) < \varepsilon$, 
then $\nabla V =0 $, $Q$-a.e.

\item[(ii)]  $\calB_\varepsilon$ constraint tight:
If $\W_2(Q, P) =  \varepsilon$, 
then either $\nabla V =0 $, $Q$-a.e. or $\exists \lambda > 0$, s.t., 
\begin{equation}\label{eq:1st-order-TR}
\nabla V +  \lambda  ({\rm I_d} -  T_Q^P) =0, \quad Q\text{-a.e.}    
\end{equation}
\end{itemize}
\end{theorem}
Note that the statement of the proposition implies that $T_Q^P = {\rm I_d} + \frac{1}{\lambda}\nabla V$, when $\lambda >0$, and otherwise $\nabla V = 0$, which takes the form of {\it complementarity condition}.

\subsection{First-order condition of \texttt{proximal problem}}

For any $\gamma >0$, the first-order condition of the Wasserstein \texttt{proximal problem} \eqref{eq:problem-W2-proximal-GD} is derived in the following proposition.

\begin{theorem}[First-order condition of \texttt{proximal problem}]\label{prop:PR}
Let $P \in \calP_2$ be fixed, 
under Assumption \ref{assump:V}, for $\gamma > 0$, suppose the problem \eqref{eq:problem-W2-proximal-GD} attains a (local) minimum at $Q \in \calP_2^r$,
then 
\begin{equation}\label{eq:1st-order-PR}
0 = \nabla V + \frac{1}{\gamma} ( {\rm I_d} - T_Q^P ),
\quad Q\text{-a.e.}
\end{equation}
\end{theorem}

\begin{remark}[Correspondence between of \texttt{LFD problem} and \texttt{proximal problem}]\label{rk:TR-PR-correspondence}
    We can see that the condition \eqref{eq:1st-order-PR} matches the first order condition \eqref{eq:1st-order-TR} (when Wasserstein ball constraint is tight) by setting $\gamma = 1/\lambda$.
\end{remark}

The $\W_2$-proximal problem has been studied in Section 10.1 of \cite{ambrosio2005gradient}, and in particular, Lemma 10.1.2 derived a first-order condition (in terms of strong subdifferential of $\varphi$) at a minimizer. 
In our case, the strong $\W_2$-differential of $\varphi$ exists at $Q$ and thus the subdifferential uniquely exists, i.e., $\partial_{\W_2} \varphi = \{  \nabla V \} $. Then the conclusion of \cite[Lemma 10.1.2]{ambrosio2005gradient} directly implies \eqref{eq:1st-order-PR}. We include a direct proof of the proposition for completeness.

The first-order condition of the $\W_2$-proximal problem allows us to prove the explicit expression of the dual form \eqref{eq:dual-LFD-G-2},
technically with small enough $\gamma$ s.t. the Moreau envelope of $V$  has unique minimizer in the $\inf_z$.

\begin{corollary}[Dual form]\label{cor:dual-form}
Let $P \in \calP_2$ be fixed, 
under Assumption \ref{assump:V}, for $0 < \gamma <  \rev{\frac{1}{L}}$, 
suppose the \texttt{proximal problem} \eqref{eq:problem-W2-proximal-GD} attains a local minimum at $Q \in \calP_2^r$.
Then,
the Moreau envelope $u(x, \gamma)$ defined in \eqref{eq:def-Moreau-V}
    is solved at an unique minimizer $z^*$ for each $x$, $Q$ is a global minimum of the \texttt{priximal problem} \eqref{eq:problem-W2-proximal-GD}, and the dual function $G$ defined in \texttt{dual form} \eqref{eq:dual-LFD-G}
     has the expression as in \eqref{eq:dual-LFD-g-gamma}.
  \end{corollary}

\begin{remark}[Interpretation of the optimal transport map]
\label{rk:back-euler}
When the optimal transport map from $P$ to $Q$ also exists, it can be interpreted as the map from $x$ to $z^*$, which solves (the unique minimizer of) the Moreau envelope as well as a Backward Euler scheme to solve the continuous-time gradient flow.
Specifically,
when $P \in \calP_2^r$, the optimal transport map $T_P^Q$ is defined $P$-a.e., and $T_Q^P \circ T_P^Q = {\rm  I_d}$, $P$-a.e.
By Theorem \ref{prop:PR}, we have \eqref{eq:1st-order-PR},  which implies that  
\begin{equation}\label{eq:expression-TPQ}
T_P^Q= {\rm  I_d}  - \gamma \nabla V \circ T_P^Q , \quad P\text{-a.e.}    
\end{equation}
By a similar argument as in the proof of Corollary \ref{cor:dual-form}, $z = T_P^Q(x) $ solves the unique minimizer of the Moreau envelope $ u(x,\gamma)= \inf_{ z}
    \left[ V(z) + \frac{1}{2\gamma } \| z - x\|^2 \right]$.
To view the map $T_P^Q$ as a Backward Euler scheme to solve the $\W_2$-proximal gradient descent:
Suppose we use $T_P^Q$ to pushforward from the current distribution $P_k = P $ to the next distribution $P_{k+1} = Q$, 
then each point $x_k$ is moved to $x_{k+1}$ by $T_P^Q$, i.e. $x_{k+1}= T_P^Q(x_k)$, then \eqref{eq:expression-TPQ} gives that 
\begin{equation}\label{eq:bwd-Euler}
x_{k+1} = x_k - \gamma \nabla V( x_{k+1} ),   
\end{equation}
which is a Backward Euler scheme to integrate the continuous-time gradient descent ODE $\dot x(t) = - \nabla V( x(t))$
with step size $\gamma$. 
\end{remark}

\section{Algorithm: Flow-DRO}\label{sec:method}

This section presents a neural network flow-based approach to solve the LFD problem by representing the optimal transport maps by ResNet blocks  \cite{he2016deep}. Our framework does not need to rely on neural networks; there can be other ways to represent the transport map (e.g., kernel representation). For high-dimensional data, with sufficient training data, neural networks tend to have competitive performance due to their expressiveness power. 
Below, in section \ref{sec:ode_param}, we first parameterize the transport map $T$ in \eqref{eq:problem-W2-T} as the solution map of a NeuralODE \citep{chen2018neural}. In section \ref{sec:flow_training}, we present the block-wise progressive training algorithm of the proposed flow model. In section \ref{sec:adv_sampler}, we explain how \flowmodel{} can be used as an adversarial generative sampler. In section \ref{sec:min_max_algo}, we propose an iterative algorithm to solve the original min-max DRO problem \eqref{eq:DRO_general} with $\mathcal{B}$ being the $\W_2$ ball around $P$.

\subsection{Flow-based neural network parameterization of transport map} \label{sec:ode_param}

Consider a density evolution (i.e., flow) $\rho(x,t)$ such that $\rho(x,0)=P$ at $t=0$, and $\rho(x,t)$ approaches $Q^*$ as $t$ increases, where $Q^*$ is the minimizer of \eqref{eq:problem-W2-proximal-GD} (unknown {\it a priori}). Below, we interchangeably refer $\rho(x,t)$ both as the marginal distribution of $x(t)$ and its corresponding density function.
Given the initial distribution $\rho(x,0)=P$, such a flow is typically non-unique. We consider when the flow is induced by an ODE of $x(t)$ in $\bR^d$:
\begin{equation}\label{eq:ode_vfield}
    \dot{x}(t) = f(x(t),t),
\end{equation}
where $x(0)\sim P$. Note that by the Liouville equation (the continuity equation) (see, e.g., \cite{JKOproof2023}), the marginal distribution $\rho(x,t)$ of $x(t)$ satisfies \[\partial_t \rho + \nabla \cdot (\rho f)=0.\]

We choose to parameterize $f(x(t),t)$ in \eqref{eq:ode_vfield} by a neural network $f(x(t),t;\theta)$ 
with trainable parameters $\theta \in \Theta$ (using continuous-time NeuralODE \cite{chen2018neural}). Assuming the flow map is within the unit interval $t\in [0,1)$, the $\theta$-parameterized solution map $T$ can be expressed as
\begin{equation}\label{eq:neural_ode_T}
    T(x;\theta)=x+\int_0^1 f(x(s'),s';\theta) ds', x(s)=x.
\end{equation}
Using \eqref{eq:neural_ode_T}, the problem of finding $T$ in \eqref{eq:problem-W2-T} thus reduces to training $\theta$ in the following problem:
\begin{equation}\label{eq:distribution_min_proximal_T_ode}
    \min_{\theta\in \Theta} \bE_{x \sim P}\left( V \circ T(x;\theta) + \frac{1}{2\gamma }  \| x - T(x;\theta)\|^2 \right).
\end{equation}

There are two main benefits of parameterizing $T$ as a flow model with parameters $\theta$. 
First, flow models are continuous in time so that we can obtain gradually transformed samples by integrating $f(x(s),s;\theta)$ over a smaller interval $[0,t)$ for $t<1$.
In practice, these gradually transformed samples can be directly compared against those obtained by other baselines, where numerical results are presented in section \ref{sec:experiments}.
Second, compared to other popular generative models such as GAN \citep{goodfellow2014generative}, the proposed flow model based on NeuralODE can be simpler and easier to train. This is because our objective \eqref{eq:distribution_min_proximal_T_ode} involves no additional discriminators to guide the training of $T(\cdot;\theta)$, and therefore no additional inner loops are required.

We also note a close connection between training $\theta$ in \eqref{eq:distribution_min_proximal_T_ode} and training continuous normalizing flow (CNF) models with transport-cost regularization \citep{finlay2020train,onken2021ot,xu2022IGNN}. 
In CNF, the problem is to train $\theta$ so that $T(\cdot;\theta)_{\#} P$ is close to the isotropic Gaussian distribution $P_Z=\mathcal{N}(0,I_d)$. To do so, the CNF objective minimizes the KL-divergence ${\rm KL}(T(\cdot;\theta)_{\#} P||P_Z)$ up to constants, upon utilizing the instantaneous change-of-variable formula \citep{chen2018neural}. To ensure a smooth and regularized flow trajectory, the transport cost $\frac{1}{2\gamma}\|T(x;\theta)-x\|_2^2$ is also commonly used as a regularization term.
Hence, the only difference between training our \flowmodel{} and a transport-regularized CNF model lies in the expression of the \textit{first term} in \eqref{eq:distribution_min_proximal_T_ode}: our \flowmodel{} minimizes $\bE_{x \sim P}\left( V \circ T(x;\theta) \right)$, which is guided by $V$ dependent on the loss function $r$ and decision function $\phi$, while CNF minimizes the KL-divergence between $T(\cdot;\theta)_{\#} P$ and $P_Z$.

\subsection{Block-wise progressive training algorithm}\label{sec:flow_training}

We propose a block-wise progressive training algorithm of minimizing \eqref{eq:distribution_min_proximal_T_ode} with respect to the network parameters $\theta$. We build on the JKO-iFlow method in \citep{xu2023invertible}, originally developed for training normalizing flows. The convergence of JKO-type $\W_2$ proximal GD for learning a generative model (a special case when the loss function is the KL divergence between the data density and the multi-variate Gaussian distribution) is shown in \cite{JKOproof2023}.

Specifically, we would learn $K$ optimal transports block-wise, where the $k$-th transport $T(\cdot,\theta_k)$ is parameterized by $\theta_k$. After training, the final optimal transport map $T_{\rm final}$ is approximated by $T_K \circ \cdots \circ T_1$ for $T_k:= T(\cdot; \hat \theta_k)$ with trained parameters $\hat{\theta}_k$; here for two mappings $T_1, \,T_2: \calX \to \calX$,  $T_2 \circ T_1(x) = T_2 (T_1(x))$. To perform block-wise progressive training, we first train $\theta_1$ using \eqref{eq:distribution_min_proximal_T_ode} with the penalty parameter $\gamma=\gamma_1$. The expectation is taken over $x(0)\sim P$, the data distribution. Using the trained parameters $\hat{\theta}_1$, we could thus compute the push-forward distribution $P(1)=(T_1)_{\#} P$. This push-forward operation is done empirically by computing $x(1) = T_1(x(0)), x(0)\sim P$ using the first trained flow block. Then, we continue training $\theta_2$ using \eqref{eq:distribution_min_proximal_T_ode} with $\gamma=\gamma_2$, where the expectation is taken over $x(1)\sim P(1)$. In general, starting at $P(0)=P$, we are able to train the $(k+1)$-th block parameters $\theta_{k+1}$ with $\gamma=\gamma_{k+1}$ given previous $k$ blocks, where the expectation is taken over $x(k)\sim P(k)$.

This leads to a block-wise progressive training scheme of the proposed \flowmodel{}, as summarized in Algorithm \ref{alg:block_wise}. 
Note that the regularization parameters $\{\gamma_k\}$ indirectly control the amount of perturbation, which is represented by the radius $\varepsilon$ in the uncertainty set \eqref{eq:W2_uncertainty_set}. Smaller choices of $\gamma$ induce greater regularization and hence allow less perturbation of $P$ by the flow model, whereas larger choices of $\gamma$ impose less regularization on the amount of transport.
Regarding the specification of these regularization parameters, we note that the desired specification varies across different problems, but setting an even choice (i.e., $\gamma_k=\gamma$) or changing by a constant factor (i.e., $\gamma_k=c\gamma_{k-1}$ for $c>0$) typically work well in practice. To further improve the empirical performance, one can also consider adaptive step size using the time reparameterization technique \citep{xu2023invertible}, which is an attempt to encourage a more even amount of $\W_2$ transport cost by individual blocks.

\begin{wrapfigure}[10]{r}{0.55\textwidth}%
 \vspace{-0.35in}
\begin{minipage}{\linewidth}
\begin{algorithm}[H]
\caption{Block-wise progressive training of \flowmodel{}}
\label{alg:block_wise}
\setstretch{1.35}
\begin{algorithmic}[1]
\REQUIRE Regularization parameters $\{\gamma_k\}_{k=1}^{K}$, 
training data $\{x_i\}\sim P$ 
\FOR{$k=1,\ldots,K$}
\STATE Optimize parameters $\theta_k$ of $T(\cdot;\theta_k)$ by minimizing the sample average approximation (SAA) version of \eqref{eq:distribution_min_proximal_T_ode} using samples mapped through previous maps $\{T(\cdot;\hat{\theta}_i)\}_{i=1}^{k-1}$, and regularization parameter $\gamma_k$ by setting $\gamma=\gamma_k$.
\ENDFOR
\ENSURE $K$ trained flow blocks $\{T(\cdot;\hat{\theta}_k)\}_{k=1}^K$.
\end{algorithmic}
\end{algorithm}
\end{minipage}
\vspace{1in}
\end{wrapfigure}
The motivation of the progressive training scheme is to improve the end-to-end training of a single complicated block, especially when we allow a large $\W_2$ ball around $P$ (e.g., due to a large $\gamma$ in \eqref{eq:distribution_min_proximal_T_ode}.) Specifically, compared to training a single large model, Algorithm \ref{alg:block_wise} with multiple blocks helps reduce the memory and computational load because each small block has simpler architecture and is easier to train. This allows larger batch sizes and more accurate numerical ODE integrators when integrating $f(x(t),t;\theta_k)$ at block $k$. 
Furthermore, we note that the proposed \flowmodel{} is adaptive: one can always terminate after training a specific number of blocks, where termination depends on the current performance measured against some application-specific metric; \rev{in the case of assuming a small $\W_2$ ball around $P$, it could also be sufficient to terminate after training a single block.} 

We also discuss the computational complexity of Algorithm \ref{alg:block_wise}. We do so in terms of the number of function evaluations of the network $f(x(t),t;\theta)$ when computing \eqref{eq:distribution_min_proximal_T_ode}, as this is the most expensive step. Suppose at block $k$, we break the integral of $f(x(t),t;\theta_k)$ over $[0,1)$ into $S\geq 1$ smaller pieces $\{[t_i,t_{i+1})\}_{i=0}^{S-1}$. Let the integral on each piece be numerically estimated by the fixed-stage Runge-Kutta fourth-order (RK4) method \citep{suli2003introduction}. As a result, it takes $\mathcal O(4NS)$ evaluation of $f(x(t),t;\theta_k)$ on $N$ samples per block. The total computation is on the order of $\mathcal O(4SKN)$ when training $K$ blocks. Note that the overall computation is linear in the number of samples and thus scalable to large datasets.

\begin{figure}[!t]
    \centering
    \includegraphics[width=0.8\textwidth]{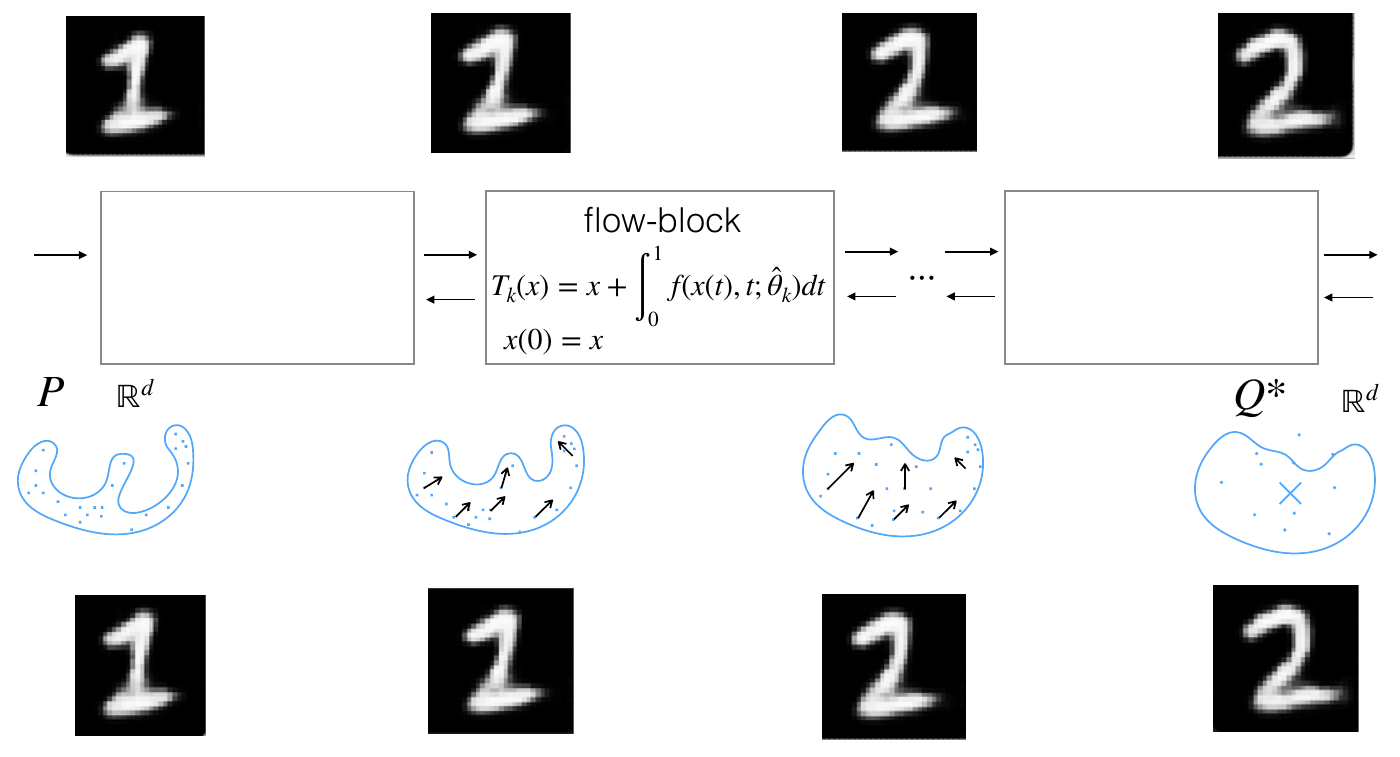}
    \caption{An illustration of the \flowmodel{} framework, which learns a sequence of invertible optimal transport maps that pushes the underlying population density $P$ to a target LFD $Q^*$; the maps are learned from finite training samples. The handwritten digits represent samples in each stage that show the gradual (continuous) transition of samples.}
    \label{fig:enter-label}
\end{figure}

\subsection{Generative model for sampling from LFD}\label{sec:adv_sampler}

We now show how \flowmodel{} can be conveniently used to generate samples from LFD. This means that we can generate samples from the worst-case distribution $Q^*$, which is found as the push-forward distribution by \flowmodel{}. Specifically, let $T_{\hat{\theta}}$ be a trained \flowmodel{} model composed of $K$ small blocks. Recall that $Q^*=(T_{\hat{\theta}})_{\#} P$ where $P$ is the data distribution. Therefore, generating samples from the LFD $Q^*$ is straightforward: one first obtains a new sample from $X\sim P$ and then computes $\tilde{X}=T_{\hat{\theta}}(X)\sim Q^*$. It remains to build a sampler for $X\sim P$. To do so, one can train an alternative generic flow model \citep{FFJORD,xu2023invertible} $T_{\rm gen}$ between $P$ and $P_Z$, where $P_Z$ is the standard multivariate Gaussian $\mathcal N(0, I_d)$, which is easy to sample from. 

As a result, we can build the sampler from LFD as $T_{\rm adv}=T_{\hat{\theta}} \circ T_{\rm gen}$. This means we can first sample from multivariate Gaussian $Z\sim P_Z$, propagate it through the generic generative model $T_{\rm gen}$ to obtain a sample from $P$, and then propagate the sample through the map $T_{\hat{\theta}}$ to obtain a sample from the LFD $Q^*$, i.e., $\tilde{X}=T_{\hat{\theta}}(T_{\rm gen}(Z)) \sim Q^*$. Figure \ref{fig:sampler} illustrates the idea.
%

\begin{wrapfigure}[8]{r}{0.45\textwidth}
    \centering
    \vspace{-0.1in}\includegraphics[width=.7\linewidth]{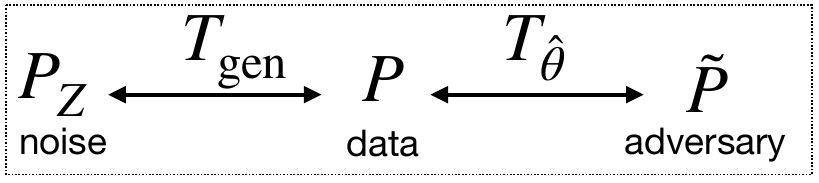}
    \caption{Construction of the proposed sampler from LFD. After training the proposed \flowmodel{} $T_{\hat{\theta}}$, we train a separate generic flow model $T_{\rm gen}$ to map between the noise distribution $P_Z$ (a multivariate Gaussian $\mathcal N(0, I_d)$) and the data distribution $P$. The full sampler $T_{\rm adv}=T_{\hat{\theta}} \circ T_{\rm gen}$.}
    \label{fig:sampler}
    \vspace{0.1in}
\end{wrapfigure}
Meanwhile, we can also perform conditional generation, which is useful for classification problems. Suppose $X=(X_{\rm sub}, Y)$ where $Y\in [C]$ is a discrete label for $X_{\rm sub}$. To generate $X_{\rm sub}$ with its corresponding $Y$, we can follow the suggestion in \citep{xu2022IGNN} to train $T_{\rm gen}$: let $P_{\rm sub}$ be the distribution of $X_{\rm sub}$. Then, train a flow model to map between $P_{\rm sub}|Y$ and $H|Y$, where $H|Y$ is a pre-specified Gaussian mixture in $\calP_2^r(\calX)$. Hence, we can sample $X_{\rm sub}$ with label $c\in [C]$ by sampling from the corresponding $H|Y=c$ and mapping through $T_{\rm gen}$. The sample $X_{\rm sub}$ can then be passed through $T_{\hat \theta}$ to get the sample with its corresponding label $c$ from LFD.

\rev{
\subsection{Iterative approach to solve min-max problem}\label{sec:min_max_algo}

\begin{wrapfigure}[10]{r}{0.55\textwidth}%
\vspace{-0.4in} 
\begin{minipage}{\linewidth}
\begin{algorithm}[H]
  \caption{\label{alg:min_max}
    Solving min-max problem using \flowmodel{}}
  \begin{algorithmic}[1]
    \REQUIRE 
    Regularization parameter $\gamma$, training data $\{x_i\} \sim P$, total iteration $N$, number of inner loops $N_{\rm inner}$.
    \FOR{$i=1,\ldots,N$}
    \STATE Optimize $T$ by minimizing the SAA of \eqref{eq:distribution_min_proximal_T_ode} for $N_{\rm inner}$ steps.
    \STATE Optimize $\phi$ by minimizing the SAA of $\cR(Q; \phi)$ defined in \eqref{eq:risk_in_expectation} over the LFD $Q=T_{\#} P$ for 1 step.
    \ENDFOR
    \ENSURE Trained models $(\hat{\phi}, \hat{T}).$
  \end{algorithmic}
\end{algorithm}
\end{minipage}
\end{wrapfigure}

While in this paper we focus on finding the LFD within the $\W_2$ ball around $P$ (i.e., solve problem \eqref{eq:problem-W2-trust-region}), we hereby propose an iterative scheme that solves the min-max problem \eqref{eq:DRO_general} that leads a pair of estimates $(\hat{\phi},\hat{Q})$. In the context of supervised learning (e.g., classification), the solution $\hat{\phi}$ denotes a predictor robust against unobserved perturbation over input data to $\hat{\phi}$.

The high-level idea is as follows. We start from the samplable data distribution $P$ and randomly initialized decision function $\phi$ and flow map $T$. We first update $T$ by minimizing \eqref{eq:distribution_min_proximal_T_ode} to find the LFD $Q=T_{\#}P$. We then update $\phi$ by minimizing the risk $\cR(Q; \phi)$ where $\cR$ is defined in \eqref{eq:risk_in_expectation}. We finally iterate the training of $T$ and $\phi$ for some number of steps until the training converges. The procedure is summarized in Algorithm \eqref{alg:min_max}. 

We further note the similarity and difference between Algorithm \ref{alg:min_max} and existing iterative DRO solvers (e.g., \citep[Algorithm 1]{sinha2018certifiable}). Both approaches iterate between finding the LFD and updating $\phi$ on samples from the LFD until convergence. The main difference lies in how the LFD is found. Our approach trains a continuous flow model $T$ whose push-forward distribution $T_{\#}P$ is the LFD. In contrast, \citep{sinha2018certifiable} solves the sample-wise LFD by iteratively moving inputs $x_i$ along the gradient $\nabla_{x} [r(x;\phi)-\frac{1}{2\gamma}\|x-x_i\|^2]$. We empirically show the benefit of our proposed flow-based approach in section \ref{sec:robust_classifier}.
}

\section{Applications}\label{sec:applications}

We consider several applications that can be formulated as DRO problems so that our proposed \flowmodel{} can be used to find the worst-case distribution.

\subsection{Adversarial learning with distributional attack}\label{subsec:adv_attack}

It has been widely known that state-of-the-art machine learning models are often adversarially vulnerable. Under small but carefully crafted perturbations to the inputs, the models can make severely wrong predictions on the adversarial examples \citep{Szegedy2014Intrigue,ijcai2021p591}. Adversarial training thus refers to the defense strategy in the clean training dataset augmented with adversarial examples, upon which retraining increases the robustness of the model on new adversarial examples. 

Finding suitable adversarial examples before retraining is a critically important step. Most methods, such as the widely-used FGSM \citep{goodfellow2015explaining} and PGD \citep{madry2018towards}, are based on the \textit{point-wise} attack. We can show that the solution of the $W_2$ trust-region problem \eqref{eq:problem-W2-trust-region} more effectively ``disrupts'' a fixed decision function $\phi$ than the solution induced by the transport map of point-wise attack. Specifically, let $\phi\in \Phi$ be a fixed decision function. Given $x \in \calX$, we define $\Tind:\mathbb R^d\rightarrow \mathbb R^d$ as the transport port associated with the following point-wise perturbation problem:
\begin{equation}\label{eq:transport_individual_attack}
    \Tind(x):=x+\delta^*_x, ~\delta^*_x=\arg\max_{\|\delta_x\|_2\leq \varepsilon} r(x+\delta_x, \phi).
\end{equation}
Denote $\QstarPoint=(\Tind)_{\#} P$ as the push-forward distribution by $\Tind$ on the data distribution $P$. Denote $\QstarDist$ as the solution of the $W_2$ trust-region problem \eqref{eq:problem-W2-trust-region}, which is our objective of interest. We thus have the following result in terms of reaching \textit{higher risk} under $\QstarDist$, the proof is in Appendix \ref{app:proof}.
\begin{proposition}\label{prop:R-dist}
    For a fixed decision function $\phi$, we have 
    $\cR(\QstarDist, \phi) \geq \cR(\QstarPoint,\phi)$.
\end{proposition}

Several works have also considered \textit{distributional} attacks on the input distribution to extend beyond point-wise attacks. For example, \citep{sinha2018certifiable} uses the Wasserstein distance to measure the difference between input and adversarial distribution. It then proposes to solve a Lagrangian penalty formulation of the distributional attack problem by stochastic gradient methods with respect to the inputs $x$. Additionally, \citep{bui2022a} shows the generality of such distributional attack methods by subsuming different point-wise attack methods under the distributional attack framework under a new Wasserstein cost function. While these works share the similar goal of solving for adversarial distributions, the proposed solutions do not solve for a continuous-time transport map as we intend to do, whose push-forward distribution of $P$ yields the worst-case distribution.

We now formally introduce the adversarial learning problem under the current DRO framework, using image classification as a canonical example \cite{goodfellow2015explaining}. Let $X=(X_{\rm img}, Y), X\sim P$ be an image-label pair with raw image $X_{\rm img}$ and its label $Y\in [C]$. The decision function $\phi$ is typically chosen as a $C$-class classifier taking $X_{\rm img}$ as the input, and the loss function $r(X,\phi)=-\log(\phi(X_{\rm img})_Y)$ is the cross-entropy loss. To find an alternative distribution $Q^*$ on which the risk is high, it is conventional to keep $Y$ the same and perturb the corresponding $X_{\rm img}$. Thus, for a given image-label distribution $P$, let $P_{\rm img} = \{X_{\rm img}: X=(X_{\rm img},Y), X\sim P\}$. As a result, the $W_2$ ball $\calB_\varepsilon(P)$ with radius $\varepsilon$ around the data distribution $P$ is defined as
\begin{equation}\label{eq:W2_ball_adv}
    \calB_\varepsilon(P) = \{Q \in \mathcal{P}_2(\calX): W_2^2(Q_{\rm img}, P_{\rm img})\leq \varepsilon^2\}.
\end{equation}
Let $\Phi$ be the set of $C$-class classifiers on images $X_{\rm img}$. The DRO problem under $\calB_\varepsilon(P)$ in \eqref{eq:W2_ball_adv} is
\begin{equation}\label{eq:dro_adversarial}
\min_{\phi \in \Phi} \max_{Q \in \calB_\varepsilon(P)} ~\mathbb{E}_{X\sim Q}[-\log (\phi(X_{\rm img})_Y)].
\end{equation}

\subsection{Robust hypothesis testing}

The goal of hypothesis testing is to develop a detector which, given two hypotheses $H_0$ and $H_1$, discriminates between the hypotheses using input data while reaching a small error probability. In practice, true data distribution often deviates from the assumed nominal distribution, so one needs to develop robust hypothesis testing procedures to improve the detector's performance. The seminal work by \citep{huber1965robust} considers the problem of using $\epsilon$-contamination sets, which are all distributions close to the base distributions in total variation. Later, \citep{levy2008robust} considers uncertainty sets under the KL-divergence and develops robust detectors for one-dimensional problems. More recently, \citep{gao2018robust} developed data-driven robust minimax detectors for non-parametric hypothesis testing, assuming the uncertainty set is a Wasserstein ball around the empirical distributions. In addition, \citep{wang2022data} derives the optimal detector by considering Sinkhorn uncertainty sets around the empirical distributions. Compared to robust detectors under Wasserstein uncertainty sets, the Sinkhorn-based method is applicable even if the test samples do not have the same support as the training samples.

We follow the notations in \citep{gao2018robust} to introduce the problem. Given data $X\in \Omega$, we test between $H_0: X\sim Q_0, Q_0 \in \calB_{0,\varepsilon}(P_0)$ and $H_1: X\sim Q_1, Q_1 \in \calB_{1,\varepsilon}(P_1)$, where $\calB_{i,\varepsilon}(P_i)$ denotes the $W_2$ ball of radius ${\varepsilon}$ as in \eqref{eq:W2_uncertainty_set} around the corresponding data distribution $P_i$.
Then, we find a measurable scalar-valued detector $\phi: \Omega \rightarrow \mathbb{R}$ to perform the hypothesis test. Specifically, for a given observation $X \in \Omega$, $\phi$ accepts $H_0$ and rejects $H_1$ whenever $\phi(X) <0$ and otherwise rejects $H_0$ and accepts $H_1$.
In this problem, the risk function $\cR((Q_0, Q_1),\phi)$ is defined to provide a convex upper bound on the sum of type-I and type-II errors. Specifically, consider a so-called \textit{generating function} $f$ that is non-negative, non-decreasing, and convex. The risk is thus defined as
\begin{equation}\label{eq:robust_hypo_test_risk}
    \cR((Q_0,Q_1),\phi)=\mathbb{E}_{x\sim Q_0}[f\circ (-\phi)(x)]+\mathbb{E}_{x\sim Q_1}[f\circ \phi(x)].
\end{equation}
Examples of the generating function $f$ to defined \eqref{eq:robust_hypo_test_risk} include $f(x)=\exp(t)$, $f(x)=\log(1+\exp(t))$, $f(x)=(t+1)^2_+$, and so on. 
As a result of $\mathcal{R}$ in \eqref{eq:robust_hypo_test_risk}, the robust hypothesis testing can be formulated as the following DRO problem
\begin{equation}\label{eq:robust_hypo_test_problem}
    \min_{\phi: \Omega\rightarrow \mathbb{R}} ~ \max_{Q_i\in \calB_{i,\varepsilon}(P_i), i=0,1} ~\mathbb{E}_{x\sim Q_0}[f\circ (-\phi)(x)]+\mathbb{E}_{x\sim Q_1}[f\circ \phi(x)].
\end{equation}
Solving the inner maximization of \eqref{eq:robust_hypo_test_problem} requires finding a pair of worst-case distributions $Q^*_0$ and $Q^*_1$. However, using the change-of-measure technique \citep[Theorem 2]{gao2018robust}, we can solve an equivalent problem of finding $Q^*$ within a $W_2$ ball round the data distribution $P=P_1+P_2$ to fit our original formulation \eqref{eq:DRO_general}.

\subsection{Differential privacy}

Established by \citep{dwork2006calibrating, dwork2006our}, differential privacy (DP) offers a structured method to measure how well individual privacy is secured in a database when collective data insights are shared as answers to the query.
In short, DP upholds robust privacy assurances by ensuring that it is nearly impossible to determine an individual's presence or absence in the database from the disclosed information.
These can be realized by introducing random perturbations to the query function output before release.

To be more precise, consider datasets $D, D' \in \mathcal D ^n$ where each consists of $n$ rows, and $\mathcal D$ is the space where each datum lies. We say $D$ and $D'$ are neighboring datasets if they differ in exactly a single element (i.e., in the record of one individual), and we denote $D \simeq D'$.
An output of the query function $q: \mathcal D^n \to \Omega$ is given based on the dataset.
A randomized mechanism $M: \mathcal D^n \to \Omega$, which maps a dataset to a random output under the probability space $(\Omega, \mathcal F, \mathbb P)$, imparts randomness to the answer to the query by perturbing $q(D)$.
Differentially private randomized mechanisms secure privacy by ensuring that the outputs of $M$ from neighboring datasets are nearly indistinguishable.

The most represented standard for DP is $(\epsilon, \delta)$-DP \citep{dwork2006our} (without causing confusion, here $\epsilon$ is not related to the radius of the uncertainty set $\varepsilon$). Given $\epsilon, \delta \ge 0$, a randomized mechanism $M$ is $(\epsilon, \delta)$-differentially private, or $(\epsilon, \delta)$-DP, if
$$
    \mathbb{P}(M(D) \in A) \le e^\epsilon \mathbb{P}(M(D') \in A) + \delta
$$
for any $D \simeq D' \in \mathcal D^n$ and $A \in \mathcal F$. When $\delta = 0$, we simply say that $M$ is $\epsilon$-DP. 
Besides, numerous variants of DP with rigorous definitions such as $f$-DP \citep{dong2022gaussianDP}, Renyi DP \citep{mironov2017renyi}, and Concentrated DP \citep{dwork2016concentrated} have been established and studied; for a comprehensive overview, see \cite{desfontaines2020sok}.

The randomized mechanisms exhibit a clear trade-off: the more they secure privacy, the more they sacrifice statistical utility \citep{alvim2012differential}. Therefore, the constant focus of research has been to design mechanisms that minimize the perturbation and thus the loss of utility (based on specific criteria such as $l_p$ cost) while ensuring a certain level of privacy.
Below, we borrow the notion of DP to conceptualize the design of a privacy protection mechanism as a DRO problem and propose the potential applicability of our \flowmodel{} as a data-dependent distributional perturbation mechanism.

DP can be understood as a hypothesis-testing problem \citep{wasserman2010statistical, kairouz2015composition, balle2020hypothesis, dong2022gaussianDP}. Consider an adversary trying to differentiate between neighboring datasets $D$ and $D'$ based on the mechanism output. In this context, the hypothesis testing problem of interest is
\begin{equation}\label{eq:dp_hypotheses}
    H_0: X \stackrel{d} = M(D) \sim Q_0 \quad {\rm vs.} \quad H_1: X \stackrel{d} = M(D') \sim Q_1
\end{equation}
where $X \in \Omega$ is a single perturbed observation. The harder this test is, the more difficult it is to distinguish between neighboring datasets, which implies that strong privacy is ensured. Consider testing \eqref{eq:dp_hypotheses} with a decision function $\phi: \Omega \to [0,1]$, and denote the type-I and type-II errors as $\alpha_\phi = \mathbb E_{X \sim Q_0} \phi(X)$ and $\beta_\phi = \mathbb E_{X \sim Q_1} (1 - \phi(X))
$. Then, a mechanism is $(\epsilon, \delta)$-DP if and only if $\alpha_\phi + e^\epsilon \beta_\phi \ge 1 - \delta$ and $e^\epsilon \alpha_\phi + \beta_\phi \ge 1 - \delta$ for any $D \simeq D'$ and decision function $\phi$ that is a deterministic function of $X$ [\citealp[Theorem 2.4]{wasserman2010statistical}; \citealp[Theorem 2.1]{kairouz2015composition}].

Now, we first set up an optimization problem with the risk function measuring indistinguishability between $Q_0$ and $Q_1$ in \eqref{eq:dp_hypotheses}, given the restricted level of perturbation and the neighboring datasets $D$ and $D'$.
Consider a risk function $\cR((Q_0, Q_1), \phi)$ representing the ease of \eqref{eq:dp_hypotheses} with a decision function $\phi: \Omega \to [0,1]$.
To ensure strong privacy with a randomized mechanism, even in the ``worst-case scenario'' with a powerful discriminator, one should make it difficult to distinguish $Q_0$ and $Q_1$ by bringing the two distributions closely together, thereby reducing the risk function.
Hence, finding such a pair of indistinguishable distributions with perturbation levels controlled by the Wasserstein-2 distance reduces to
\begin{equation}\label{eq:dp_formulation}
    \min_{Q_i \in \mathcal{B}_{i,\varepsilon}(P_i), i=0,1} \max_{\phi \in \Phi} ~\mathcal R((Q_0,Q_1), \phi)
\end{equation}
where $\mathcal{B}_{i,\varepsilon}(P_i)$ denotes the $\calW_2$ ball of radius $\varepsilon$ as in \eqref{eq:W2_uncertainty_set} around the corresponding data distribution $P_i$.

In this context, the risk function can be chosen based on which measure reflects the indistinguishability of outputs from neighboring datasets.
For instance, under the $f$-DP criterion, one must first consider the most powerful test for a given level $\alpha$: the decision function that minimizes $\beta_\phi$. The corresponding problem is formulated as finding $\min_\phi \, \beta_\phi$ subject to $\alpha_\phi \le \alpha$.
Therefore, using the Lagrange multiplier and the change-of-measure technique, our DRO formulation \eqref{eq:dp_formulation} becomes
\begin{equation}\label{eq:f-DP_formulation}
    \max_{Q_i \in \mathcal{B}_{i,\varepsilon}(P_i), i=0,1} \min_\phi \max_{\lambda \ge 0} -\mathbb E_{x \sim Q_0 + Q_1} \left[ \frac{d Q_1}{d (Q_0 + Q_1)}[x] \phi(x) - \lambda \left( \frac{d Q_0}{d (Q_0 + Q_1)}[x] \phi(x) - \alpha \right) \right].
\end{equation}
In our experiments, we will use $\alpha_\phi$ and $\beta_\phi$ as performance measures by replacing them with sample average approximations.

The conventional and straightforward method to privatize a query function is to apply a calibrated additive noise. In this case, the $i$-th uncertainty set in \eqref{eq:dp_formulation} is $\mathcal{B}_{i,\varepsilon}(P_i)=\{Q_i: M(D) \sim Q_i, M(D)=q(D)+\xi_i, q(D) \sim P_i, D \in \mathcal D^n\}$, where $\xi_i$ with $\mathbb{E}\|\xi_i\|_2\leq \varepsilon$ is a random noise following certain distributions from a specific family.
We call such a mechanism that adds noise of a certain distribution an \textit{additive perturbation mechanism} (APM). 
Typical noise distributions used in APM include the Laplace \citep{dwork2006calibrating} and Gaussian distributions \citep{dwork2014algorithmic}.

In contrast, based on the formulation \eqref{eq:dp_formulation}, we aim to introduce distributional perturbation with our \flowmodel{} to provide a more flexible mechanism. Consequently, we want to ensure the mechanism outputs are indistinguishable with less perturbation than additive mechanisms. We refer to the corresponding mechanism as the \textit{distributional perturbation mechanism} (DPM) and illustrate its comparison with APM in Figure \ref{fig:APM_vs_DPM}. 
We remark that the proposed \flowmodel{} allows the DPM to apply an arbitrary amount of perturbation to the original distribution of queries. Thus, we can apply DPM at arbitrary precision by controlling the perturbation to satisfy the privacy constraints with reasonable utility.

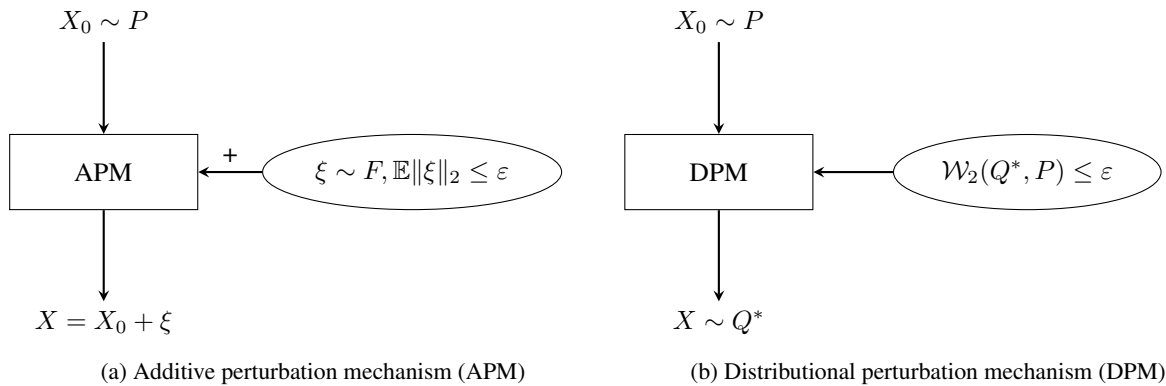
\begin{figure}[t!]
    \centering
    \begin{minipage}{0.49\textwidth}
        \begin{tikzpicture}[node distance=2cm]
            \node (in) [] {$X_0 \sim P$};
            \node (mech) [rect, below of=in] {APM};
            \node (add_noise) [elli, right of=mech, xshift=2.1cm] {$\xi \sim F, \mathbb{E}\|\xi\|_2\leq \varepsilon$};
            \node (out) [below of=mech] {$X = X_0 + \xi$};
            
            \draw [arrow] (in) -- (mech);
            \draw [arrow] (add_noise) -- node[anchor=south] {+} (mech);
            \draw [arrow] (mech) -- (out);
        \end{tikzpicture}
        \subcaption{Additive perturbation mechanism (APM)}
    \end{minipage}
    \begin{minipage}{0.49\textwidth}
        \begin{tikzpicture}[node distance=2cm]
            \node (in) [] {$X_0 \sim P$};
            \node (mech) [rect, below of=in] {DPM};
            \node (dist_pert) [elli, right of=mech, xshift=2.1cm] {$\calW_2(Q^*, P) \le \varepsilon$};
            \node (out) [below of=mech] {$X \sim Q^*$};
            
            \draw [arrow] (in) -- (mech);
            \draw [arrow] (dist_pert) -- (mech);
            \draw [arrow] (mech) -- (out);
        \end{tikzpicture}
        \subcaption{Distributional perturbation mechanism (DPM)}
    \end{minipage}
    \caption{Comparison between APM and DPM for differential privacy. APM adds random noises $\xi$ \textit{independently} to queries, whereas DPM (through the use of proposed \flowmodel{}) considers the data distribution $P$ defined over \textit{all} queries to find a worst-case distribution $Q^*$ within $\calB_{\varepsilon}(P)$.}
    \label{fig:APM_vs_DPM}
\end{figure}

\section{Numerical Examples}\label{sec:experiments}

We conduct experiments to examine the effectiveness of \flowmodel{} on high-dimensional data. First, in section \ref{sec:expr_vs_dro}, we compare our proposed \flowmodel{} with existing DRO methods to solve robust hypothesis testing problems and train robust classifiers. Then, in section \ref{expr:adversarial}, we use \flowmodel{} to perform the adversarial attack on pre-trained image classifiers and compare against existing point-wise attack methods. In section \ref{expr:dp}, we use \flowmodel{} as the DPM in differential privacy settings and compare it against APM under different noise distribution specifications. In all examples of finding the LFD, we assume the decision function $\phi$ is pre-trained on the data distribution $P$ and fixed, so the goal is to find the worst-case distribution $Q^*\in \calB_\varepsilon(P)$ defined in \eqref{eq:W2_uncertainty_set} and compare what \flowmodel{} found against that by other methods. Code is available on \url{https://github.com/hamrel-cxu/FlowDRO}.

\subsection{Comparison with existing DRO methods}\label{sec:expr_vs_dro}

We first compare the proposed \flowmodel{} in Algorithm \ref{alg:block_wise} against WDRO \citep{xie2021robust} in finding LFD. We then compare the DRO solver \ref{alg:min_max} against the Wasserstein Robust Method (WRM) \citep{sinha2018certifiable}. \rev{Further details of the experiments are in Appendix \ref{app:dro_compare}.}

\subsubsection{Finding LFD}

\begin{figure}[h]
    \begin{minipage}{0.48\linewidth}
        \includegraphics[width=\linewidth]{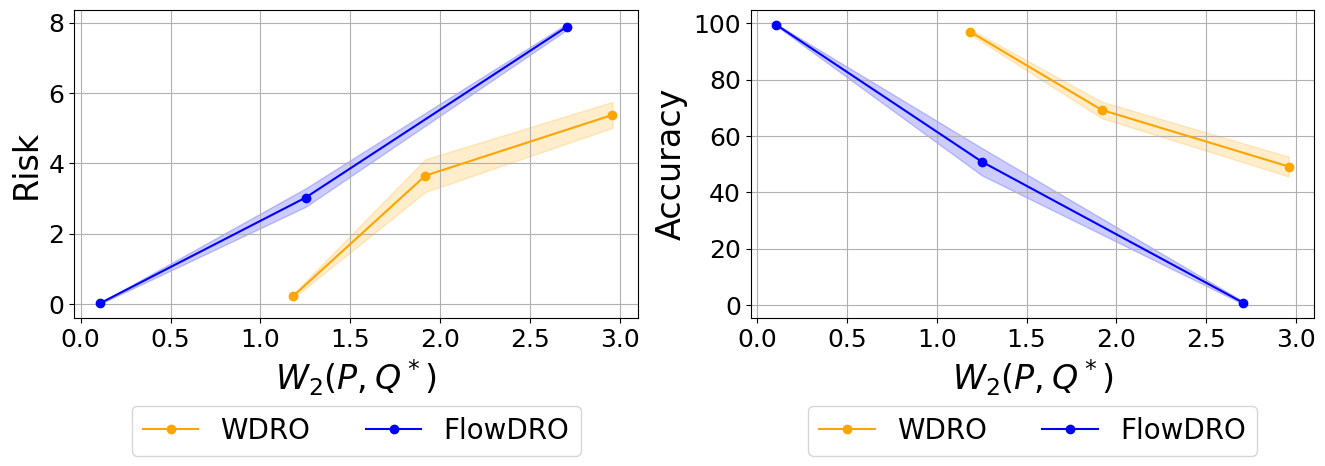}
        \subcaption{Training size $n=200$}
    \end{minipage}
    \begin{minipage}{0.48\linewidth}
        \includegraphics[width=\linewidth]{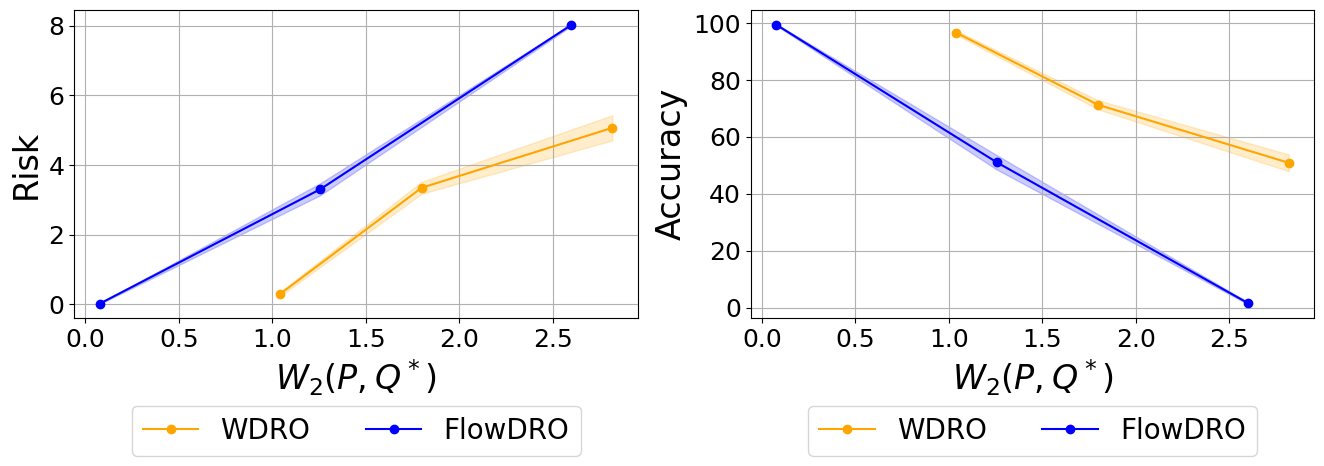}
        \subcaption{Training size $n=500$}
    \end{minipage}
    \caption{Test performance of a pre-trained MNIST classifier $\phi$ on LFDs $Q^*$ of binary MNIST digits; the higher the risk, and the lower the accuracy, the better, meaning we have achieved a more effective attack for the same amount of Wasserstein-2 perturbation from the nominal distribution. WDRO and \flowmodel{} find the LFDs within different $W_2$ balls around $P$, which consists of $n$ training data from MNIST. The empirical $W_2$ distances upon solving the earth moving distance between $P$ and $Q^*$ are shown on the $x$-axis.}
    \label{fig:wdro_results}
\end{figure}

We consider binary MNIST digits from classes 0 and 8 as an example. Given a pre-trained CNN classifier, the goal is to find the LFD around the original digits. We measure the effectiveness of the LFD according to how the pre-trained classifier performs: the found LFD is more effective if, at the same level of $W_2$ perturbation, the classifier reaches a lower test prediction accuracy and a higher test risk on samples from that LFD.

Figure \ref{fig:wdro_results} shows the test risk and accuracy of the pre-trained classifier on the LFDs obtained by WDRO and \flowmodel{}. We see that compared to WDRO, our proposed \flowmodel{} finds more effective LFDs with the same or even smaller \textit{budget}, which is measured as the empirical $W_2$ distance between $P$ and $Q^*$, the found LFD. The benefit of \flowmodel{} holds both small ($n=200$) and large ($n=500$) sample sizes.

\rev{
\subsubsection{Training robust classifiers}\label{sec:robust_classifier}
We consider both MNIST digits and CIFAR10 images as examples. The goal is to train a robust classifier $\phi$ so that when test images are attacked by PGD under $\ell_p$ norms, the classifier can defend against such attacks by incurring a small classification error. Hence, one classifier is more robust than another when it reaches a smaller classification error on the same set of attacked test images.

Figure \ref{fig:robust_CNN} shows test errors by robust classifiers trained via WRM \citep{sinha2018certifiable} and via our proposed \textit{flow robust method} (FRM) in Algorithm \ref{alg:min_max}. We can see that for small attacks (for instance, when the attack budget is below 0.2), our method is slightly better than WRM (except for one case of CIFAR-10 binary class). However, for ``higher'' attacks, our FRM shows significantly better performance. The experiments show the effectiveness of our methods in obtaining an overall more robust classifier.

\begin{figure}[!t]
    \begin{minipage}{0.48\linewidth}
        \includegraphics[width=\linewidth]{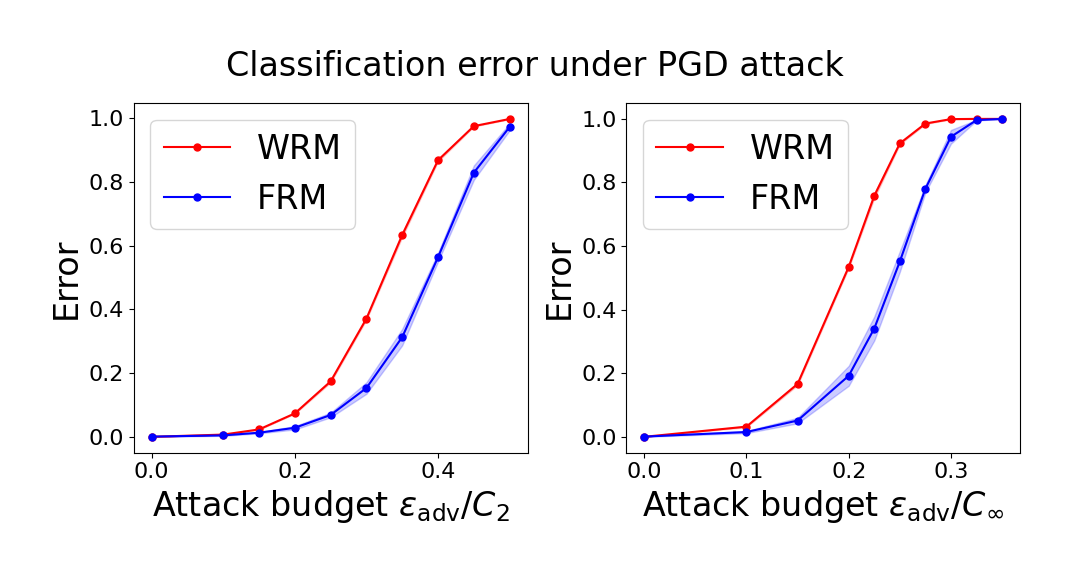}
        \subcaption{Results on binary-class MNIST digits}
    \end{minipage}
    \begin{minipage}{0.48\linewidth}
        \includegraphics[width=\linewidth]{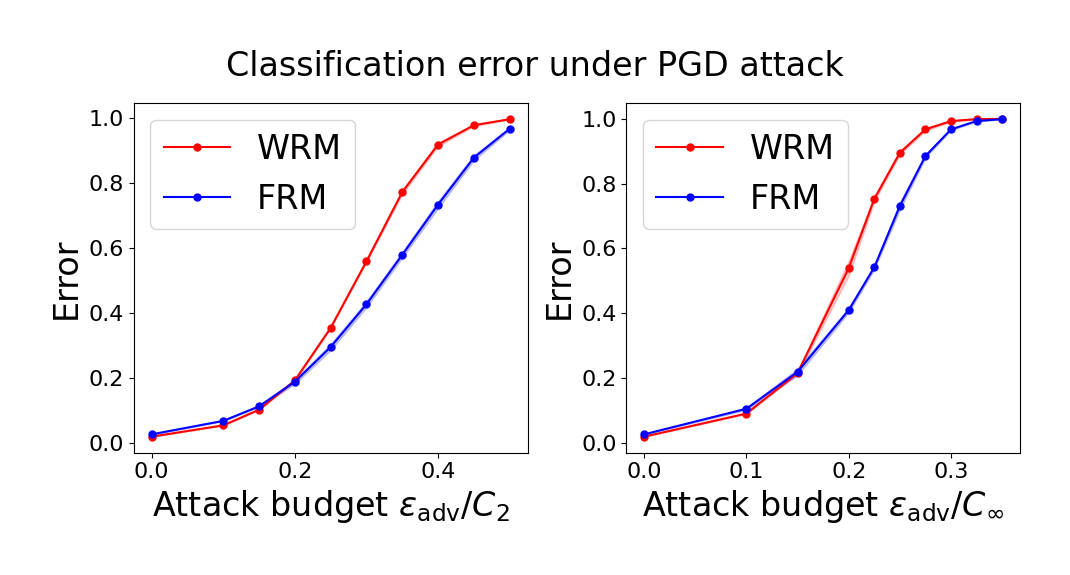}
        \subcaption{Results on full MNIST digits}
    \end{minipage}
    
    \begin{minipage}{0.48\linewidth}
        \includegraphics[width=\linewidth]{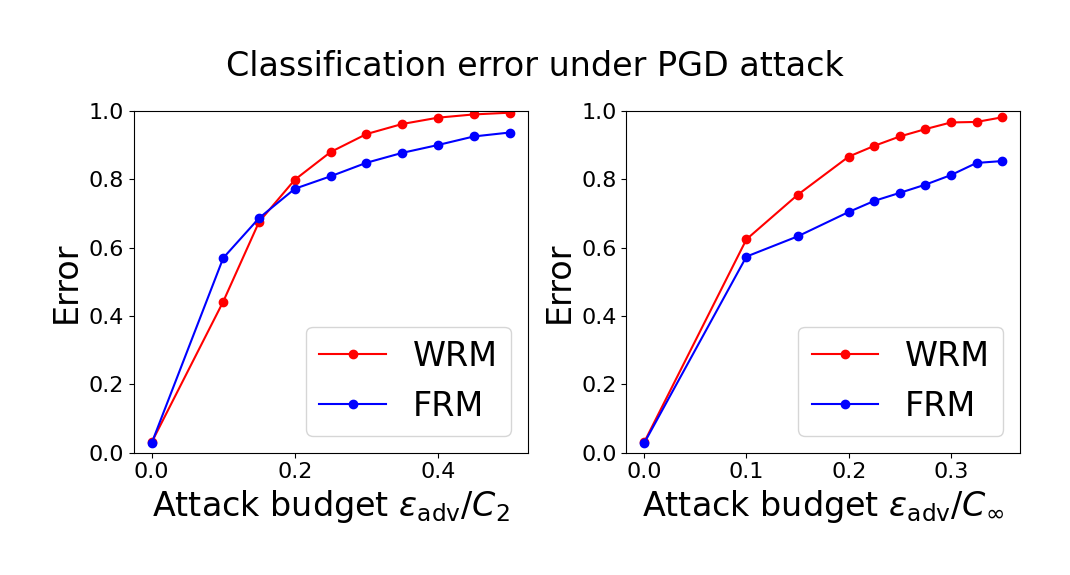}
        \subcaption{Results on binary-class CIFAR10 images}
    \end{minipage}
    \begin{minipage}{0.48\linewidth}
        \includegraphics[width=\linewidth]{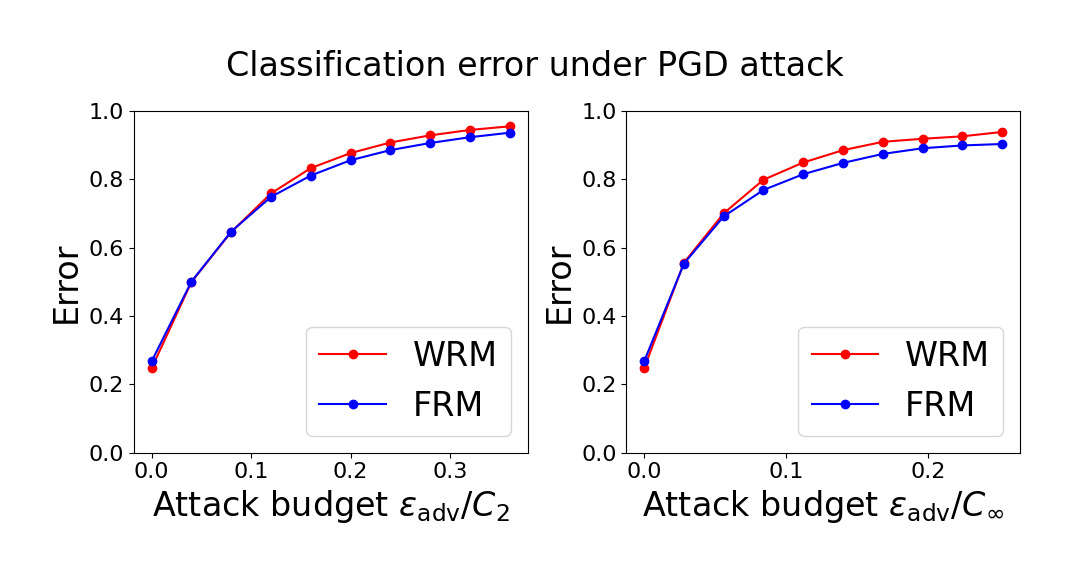}
        \subcaption{Results on full CIFAR10 images}
    \end{minipage}
    \caption{Test classification error of robust classifiers on test data attacked by PGD under $\ell_2$ and $\ell_{\infty}$ norm. The lower the error, the better, meaning we have achieved a more robust algorithm at the same attack budget. The robust classifiers are trained via solving the DRO problem using WRM \citep{sinha2018certifiable} and FRM (ours in Algorithm \ref{alg:min_max}). The binary classification results are for two randomly selected classes out of ten. The attack budget on the x-axis denotes the $\ell_p$ norm between raw and PGD-attacked test data as a fraction of $C_p$, the $\ell_p$ norm of raw test data. 
    }
    \label{fig:robust_CNN}
\end{figure}
}

\subsection{Adversarial distributional attack}\label{expr:adversarial}

We consider two sets of experiments in this section. The first example finds the distributional perturbation of CIFAR-10 images by \flowmodel{}, where we compare the effectiveness of our distributional attack against the widely-used projected gradient descent (PGD) baselines under $\ell_2$ and $\ell_{\infty}$ perturbation \citep{madry2018towards}. The second example finds the distributional perturbation of MNIST digits by \flowmodel{}. 

\subsubsection{CIFAR10 against point-wise attacks} We describe the setup, introduce the comparison metrics, and present the comparative results.

\begin{table}[!b]
    \centering
    \caption{Risk and accuracy of a pre-trained VGG-16 classifier $\phi$ on clean test data distribution $P_{\rm test}$ and adversarially perturbed data distribution $Q^*_{\rm test}$ by \flowmodel{} and by PGD under $\ell_2$ and $\ell_{\infty}$ perturbation. For a fair comparison, we control the same amount of $\ell_2$ perturbation on the test distribution by different attackers.}
    \label{tab:adv_compare}
    \setstretch{1.2}
    \resizebox{0.9\textwidth}{!}{%
        \begin{tabular}{c|c|c|c|c}
            \hline
             & Clean data & Attack by FlowDRO & Attack by PGD-$\ell_{2}$ & Attack by PGD-$\ell_{\infty}$ \\
            \hline
            Risk of $\phi$ in \eqref{eq:adv_attack_risk} & 2.03 & 32.32 & 6.22 & 10.51 \\
            Accuracy of $\phi$ in \eqref{eq:adv_attack_accu}& 87.02 & 24.22 & 61.44 & 41.57 \\
            \hline
        \end{tabular}}
\end{table}

\vspace{0.1in}
\noindent \textit{Setup.} Given a pre-trained image classifier $\phi$ and a test image $X_{\rm test, \rm img}$ with labels $Y_{\rm test}$, the goal of adversarial attack as introduced in section \ref{subsec:adv_attack} is to find a perturbed image $\tilde{X}_{\rm test, \rm img}$ of $X_{\rm test, \rm img}$ so that $\phi$ makes an incorrect classification on image $\tilde{X}_{\rm test, \rm img}$. For this task, instead of performing a point-wise attack given individual $X_{\rm test, \rm img}$, our \flowmodel{} finds a continuous flow $T$ that gradually transports the distribution of raw images to an adversarial worst-case distribution, on which the classifier $\phi$ makes incorrect classification and induces high classification losses. 

Regarding training specifics, we pre-train a VGG-16 classifier \citep{vgg16} $\phi$ with cross-entropy loss on the set of clean CIFAR-10 images, and then train three \flowmodel{} flow blocks with $\gamma_k\equiv10$ using Algorithm \ref{alg:block_wise}. We train \flowmodel{} in the latent space of a variational auto-encoder as proposed by \citep{esser2021taming}, where the latent space dimension $d=192.$ 
The architecture of the \flowmodel{} model on CIFAR10 consists of convolutional layers of \texttt{3-128-128-256}, followed by convolutional transpose layers of \texttt{256-128-128-3}. The kernel sizes and strides in the CIFAR10 attacker are \texttt{3-3-3-3-4-3} and \texttt{1-2-1-1-2-1}. We use the softplus activation with $\beta=20$. 
Each block is trained for 15 epochs using a batch size of 500, with the Adam optimizer \citep{Adam} with a constant learning rate of 1e-3.

\begin{figure}[!t]
    \centering
    \begin{minipage}{0.49\textwidth}
        \centering
        \includegraphics[width=\linewidth]{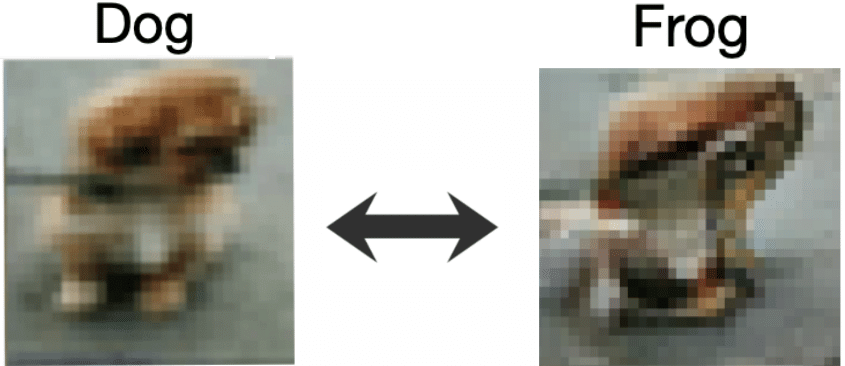}
        \includegraphics[width=\linewidth]{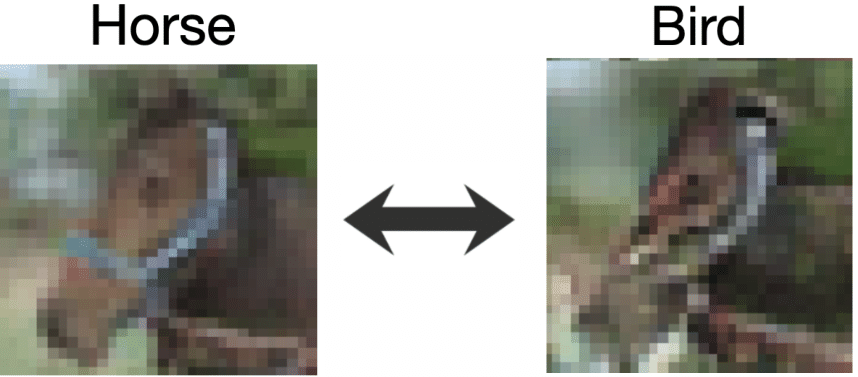}
        \includegraphics[width=\linewidth]{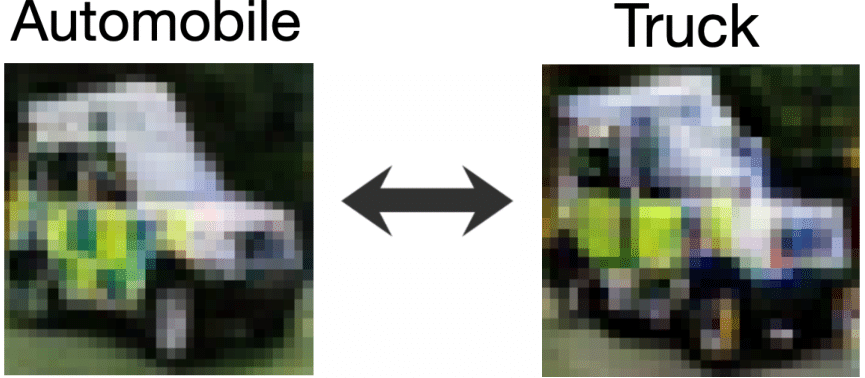}
        \subcaption{Raw (left) \& adversarial (right) samples by \flowmodel{}}
    \end{minipage}
    \begin{minipage}{0.485\textwidth}
        \centering
        \includegraphics[width=\linewidth]{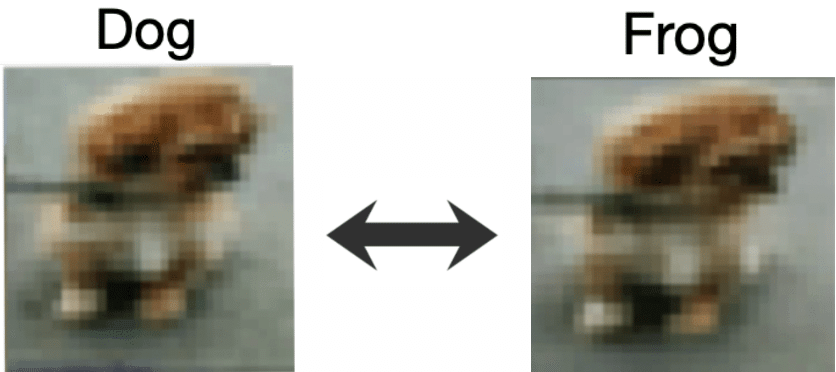}
        \includegraphics[width=\linewidth]{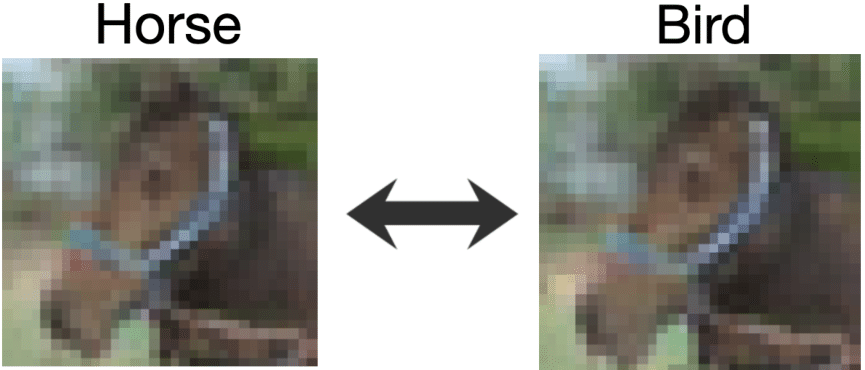}
        \includegraphics[width=\linewidth]{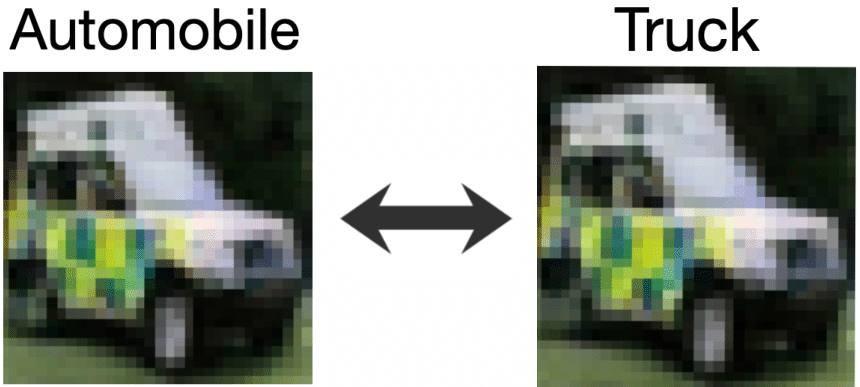}
        \subcaption{Raw (left) \& adversarial (right) samples by PGD-$\ell_2$}
    \end{minipage}
    \caption{Raw and adversarial samples found by \flowmodel{} and by PGD-$\ell_2$. Captions show prediction by the pre-trained classifier $\phi$ on raw input images $X_{\rm test, \rm img}$ before attack and adversarial samples $\tilde{X}_{\rm test, \rm img}$ after attack. \flowmodel{} results in more meaningful contextual changes in the raw images.}
    \label{fig:cifar10_changes}
\end{figure}

\vspace{0.1in}
\noindent \textit{Comparison metric.} Denote $P_{\rm test}$ as the distribution of raw image-label pairs in the test set. We evaluate the effectiveness of adversarial attack by \flowmodel{} and PGD on the pre-trained classifier $\phi$. Specifically, given test images $X_{\rm test, \rm img}$ with labels $Y_{\rm test}$, we find adversarial samples $\tilde{X}_{\rm test, \rm img}$ using different attack mechanisms, where we fix \textit{identical} amounts of $\ell_2$ perturbation measured by $\mathbb{E}_{X_{\rm test}\sim P_{\rm test}}\|X_{\rm test, \rm img}-\tilde{X}_{\rm test, \rm img}\|^2$ to ensure a fair comparison. Then, given $Q^*_{\rm test}$ defined by the set of $(\tilde{X}_{\rm test, \rm img}, Y_{\rm test})$, we evaluate $\phi$ on $Q^*_{\rm test}$ based on the sample average of
\begin{align}
    \cR(Q^*_{\rm test}, \phi) & = \mathbb{E}_{X\sim Q^*_{\rm test}}[r(\phi(X_{\rm img}),Y)], \label{eq:adv_attack_risk}\\
    \text{Accurcy}(Q^*_{\rm test}, \phi) &= \mathbb{E}_{X\sim Q^*_{\rm test}}[100\cdot\mathbbm{1}(Y=\arg\max_j \phi(X_{\rm img})_j)].\label{eq:adv_attack_accu}
\end{align}
Hence, under the same amount of perturbation to find $Q^*_{\rm test}$, a higher risk \eqref{eq:adv_attack_risk} and a lower accuracy \eqref{eq:adv_attack_accu} indicates a more effective adversarial attack on $\phi$. We also evaluate \eqref{eq:adv_attack_risk} and \eqref{eq:adv_attack_accu} on the clean test data distribution $P_{\rm test}$ for reference.

\vspace{0.1in}
\noindent \textit{Results.} Table \ref{tab:adv_compare} quantitatively compares the risk and accuracy of the pre-trained classifier $\phi$ on CIFAR10. We notice that under the same amount of $\ell_2$ perturbation between raw and perturbed images, $\phi$ on the adversarial distribution found by \flowmodel{} yields significantly larger risk and lower accuracy. Hence, we conclude that \flowmodel{} performs much more effective attacks than the PDG baselines. Meanwhile, Figure \ref{fig:cifar10_changes} visualizes the qualitative changes to test images $X_{\rm test, \rm img}$ by \flowmodel{} and PGD, where the proposed \flowmodel{} also induces more meaningful contextual changes to the input image. Lastly, Figure \ref{fig:cifar10_trajectory} visualizes the gradual changes of $X_{\rm test, \rm img}$ over blocks and their integration steps by \flowmodel{}, demonstrating the continuous deformation by our trained flow model on test images $X_{\rm test, \rm img}$.

\begin{figure}[!t]
    \centering
    \includegraphics[width=\textwidth]{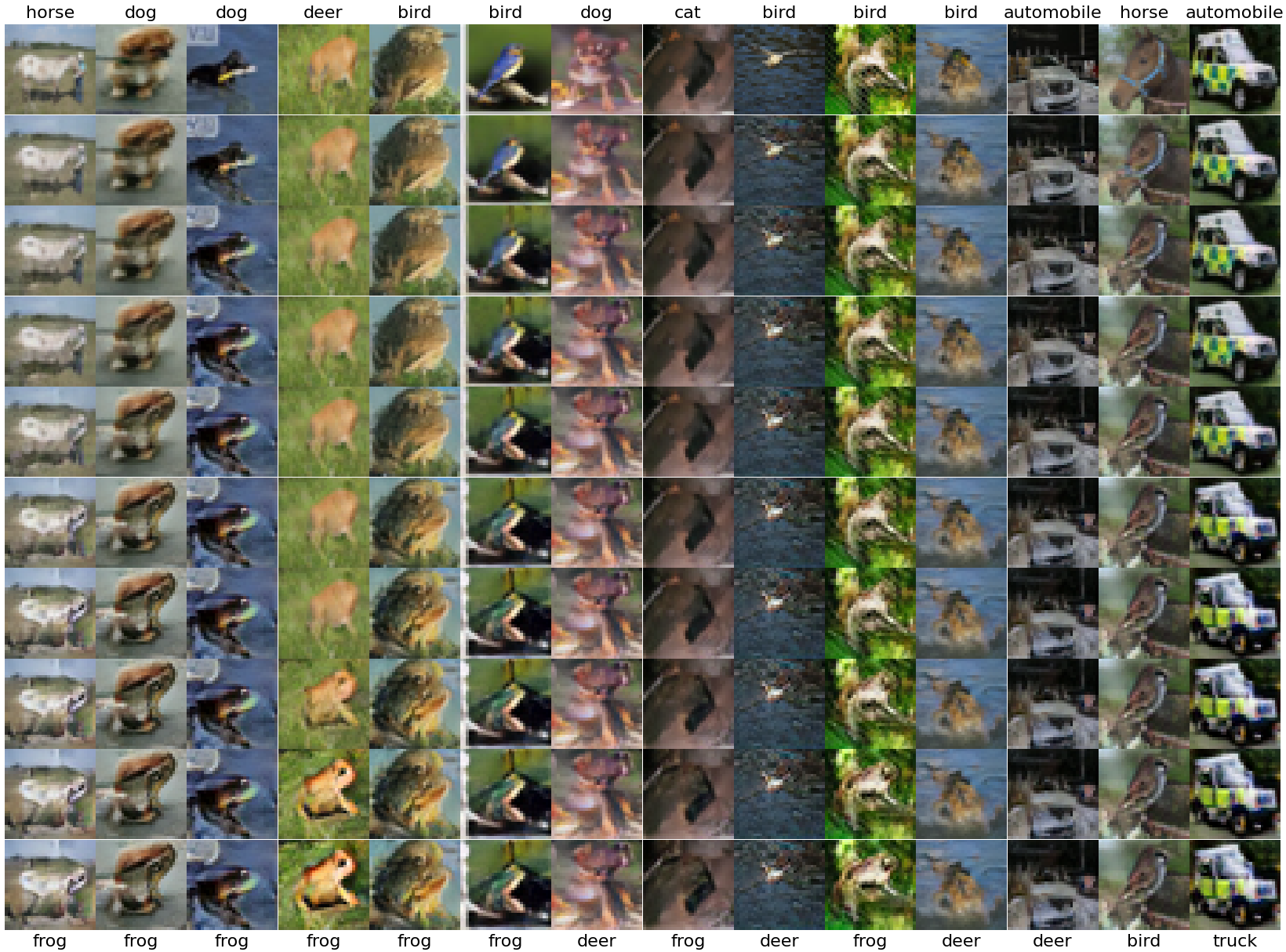}
    \caption{Trajectory of \flowmodel{} adversarial attacks on different $X_{\rm test, \rm img}$ (shown as columns) to $\tilde{X}_{\rm test, \rm img}$. We visualize the changes as rows over three \flowmodel{} blocks, each of which breaks $[0,1)$ into three evenly spaced sub-intervals, resulting in nine integration steps along the perturbation trajectory. Captions on the top and bottom indicate predictions by the pre-trained $\phi$ on raw $X_{\rm test, \rm img}$ and final perturbed adversarial $\tilde{X}_{\rm test, \rm img}$.}
    \label{fig:cifar10_trajectory}
\end{figure}

\subsubsection{MNIST trajectory illustration} \label{expr:mnist_adversarial}

We now apply \flowmodel{} on finding the worst-case distribution, given a pre-trained LeNet classifier \citep{LENET} $\phi$. In this example, we focus on providing more insights into the behavior of \flowmodel{} without comparing it against other baselines. We train \flowmodel{} using Algorithm \ref{alg:block_wise} for three blocks with $\gamma_k \equiv 1$, on the latent space of an auto-encoder with latent dimension $d=16$. The architecture of the flow model consists of fully connected layers of \texttt{d-256-256-d} with softplus activation. 

Figure \ref{fig:mnist_trajectory} visualizes the gradual and smooth perturbation of test images $X_{\rm test, \rm img}$ by \flowmodel{}. We notice the cost-effectiveness and interpretability of \flowmodel{}. First, the T-SNE embedding in Figure \ref{fig:mnist_tsne} shows that \flowmodel{} tends to push digits around the \textit{boundary} of certain digit clouds to that of other digit clouds, as such changes take the least amount of transport cost but can likely induce a great increase of the classification loss by $\phi$. Second, changes in the pixel space in Figure \ref{fig:mnist_gradual} show that visible perturbation is mostly applied to the foreground of the image (i.e., actual digits), as changes in the foreground tend to have a higher impact on the classification by $\phi$.

\begin{figure}[!t]
    \centering
    \begin{minipage}{0.575\linewidth}
        \includegraphics[width=\linewidth]{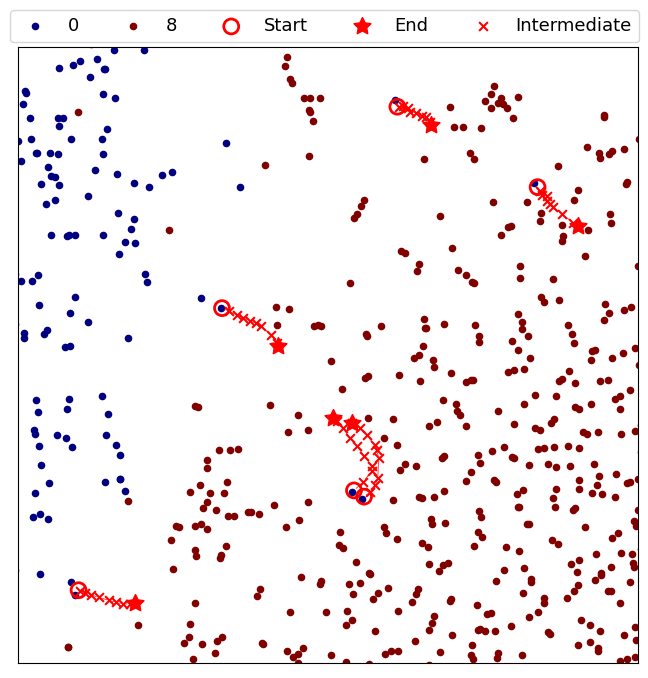}
        \subcaption{Digit deformation trajectories under T-SNE embedding}
        \label{fig:mnist_tsne}
    \end{minipage}
    \begin{minipage}{0.395\linewidth}
        \includegraphics[width=\linewidth]{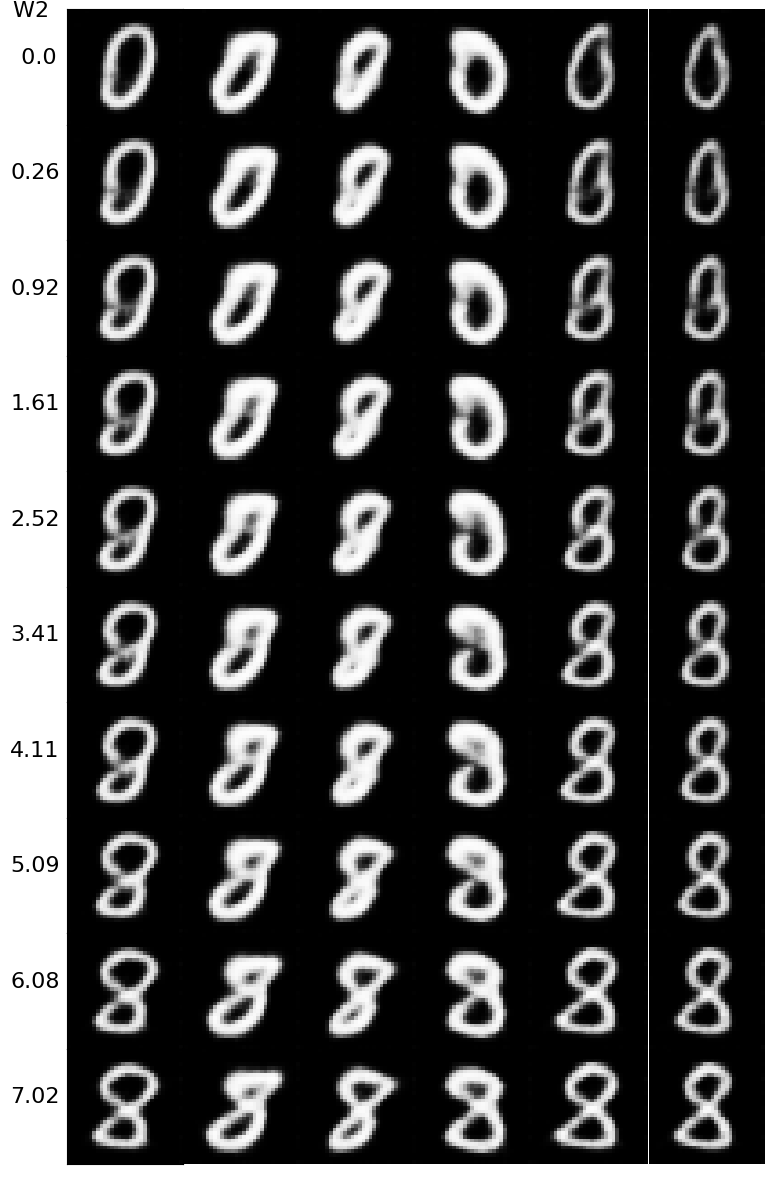}
        \subcaption{Transition of $X_{\rm test, \rm img}$ in pixel space}
        \label{fig:mnist_gradual}
    \end{minipage}
    \caption{\flowmodel{} perturbation of MNIST digits over blocks and integration steps. Figure (a) visualizes the perturbation trajectories from digits 0 to 8 under 2D T-SNE embedding. Figure (b) shows the trajectory in pixel space, along with the corresponding $W_2$ distance between original and perturbed images over integration steps.}
    \label{fig:mnist_trajectory}
\end{figure}

\subsection{Data-driven differential privacy}\label{expr:dp}

This section demonstrates the benefit of our \flowmodel{} DPM in privacy protection. We specifically focus on the examples of image recognition based on MNIST, where the decision function $\phi$ is specified as pre-trained classifiers. We mainly compare DPM against two APM baselines: APM under Gaussian noise (APM-G) and APM under Laplacian noise (APM-L).

\subsubsection{MNIST raw digit classification}\label{expr:dp_raw}

\begin{figure}[!t]
    \begin{minipage}{\textwidth}
        \includegraphics[width=\linewidth]{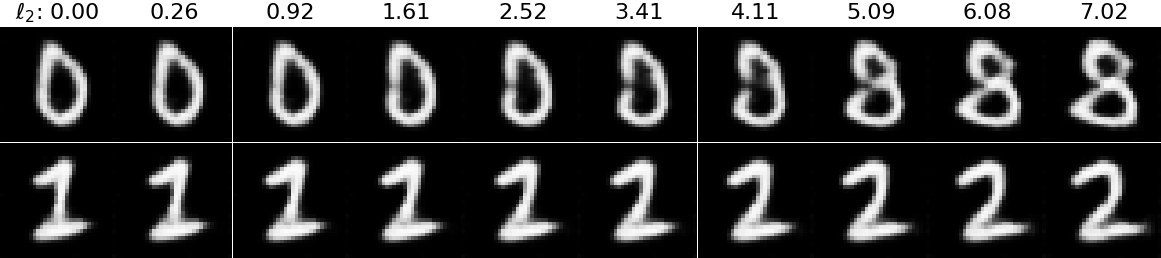}
        \subcaption{DPM by \flowmodel{}}
    \end{minipage}
    \begin{minipage}{\textwidth}
        \includegraphics[width=\linewidth]{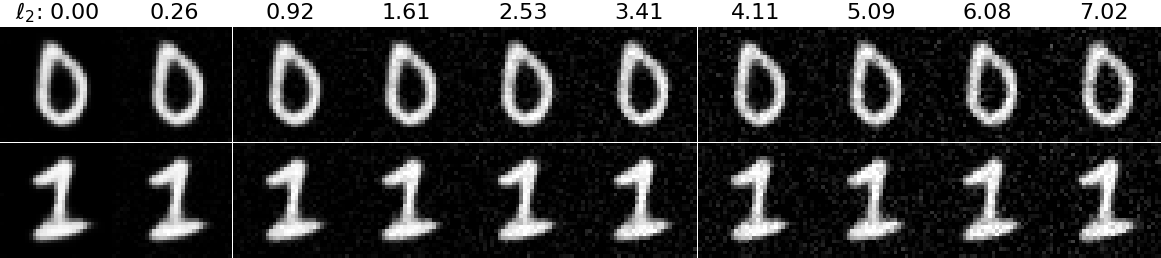}
        \subcaption{APM-G (additive Gaussian)}
    \end{minipage}
    \begin{minipage}{\textwidth}
        \includegraphics[width=\linewidth]{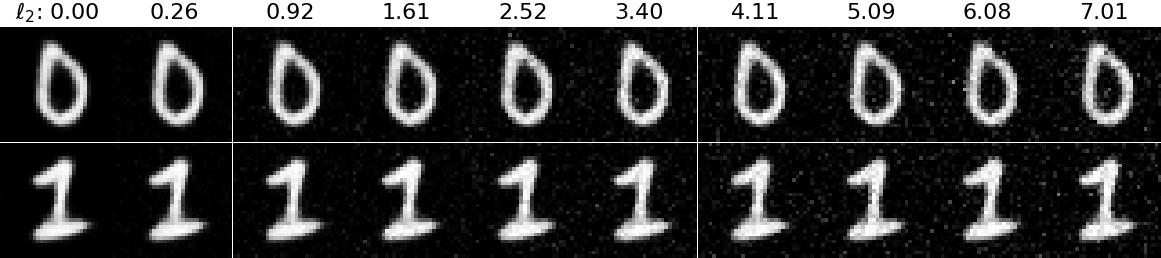}
        \subcaption{APM-L (additive Laplacian)}
    \end{minipage}
    \begin{minipage}{\textwidth}
        \includegraphics[width=\linewidth]{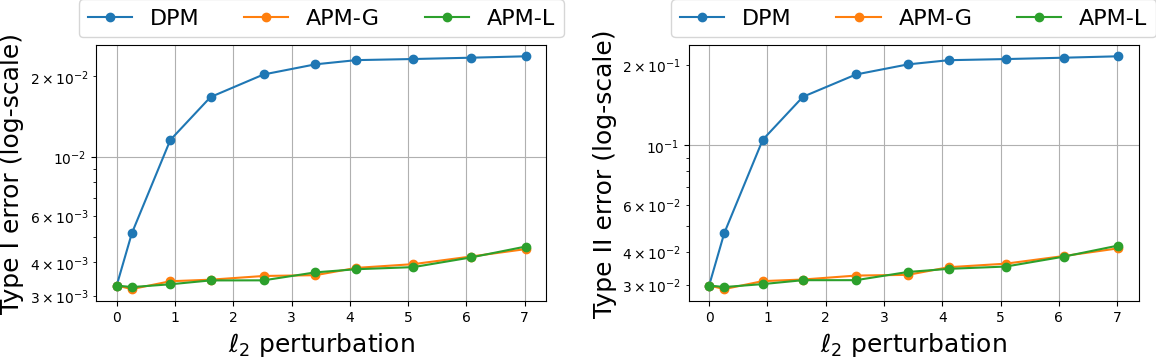}
        \subcaption{Type-I and type-II errors by DPM, APM-G, and APM-L.}
    \end{minipage}
    \caption{Differential privacy example of raw MNIST digit recognition. We control the $\ell_2$ perturbation amount by DPM, APM-G, and APM-L to be identical for a fair comparison. Figures (a)-(c) visualize privacy-protected queries $M_{\varepsilon}(D)$ by DPM, APM-G, and APM-L over different $\varepsilon$. Figure (d) examines the corresponding type-I and type-II errors defined in \eqref{dp:type_I_II_ave} by these mechanisms.}
    \label{fig:dp_raw_mnist}
\end{figure}

We first describe the precise DP setup and comparison metrics and then show the results against the baselines. This example directly follows from the adversarial attack example on MNIST in section \ref{expr:mnist_adversarial}. Specifically, the decision function $\phi$ is a pre-trained LeNet classifier on raw MNIST digits with ten classes, and we train three continuous flow blocks on the class of all digits using Algorithm \ref{alg:block_wise}.

\vspace{0.1in}
\noindent \textit{DP setup.} We describe the following components of a hypothesis-testing-based DP framework: 
(1) the definition of neighboring datasets $D$ and $D'$, where the two datasets differ in exactly one record; 
(2) the choice of the query function $q$ taking the datasets as inputs; 
(3) the privacy-protection randomized mechanism $M_{\varepsilon}$ applied to queries, under the constraint that $Q\in \calB_\varepsilon(P)$ for $\calB_\varepsilon(P)$ defined in \eqref{eq:W2_uncertainty_set};
(4) the hypothesis testing problem with the decision function $\phi$ to distinguish between $D$ and $D'$. Notation-wise, we assume $X \sim P$ is a pair $X=(X_{\rm img},Y)$, where $X_{\rm img}$ is the raw image and $Y \in \{0,\ldots,9\}$ is the corresponding label.

For (1), we let each dataset $D$ contain one image-label pair $X\sim P$ so that two datasets $D$ and $D'$ are naturally neighbors in terms of $X$. In other words, $D$ and $D'$ contain digits from the same class or different classes. 
For (2), given $D=\{X\}$, we let the query function $q(D)=X_{\rm img}$ so that it returns the raw image of the image-label pair $X$. 
For (3), the privacy-protection randomized mechanism $M_{\varepsilon}$ either applies our trained \flowmodel{} model to $q(D)$ or adds random Gaussian or Laplacian noises to $q(D)$, both under the pre-specified amount of perturbation controlled by $\varepsilon$.
For (4), given a privacy-protected image $M_{\varepsilon}(D)$ with the (unknown) label $Y$, we consider the following sets of hypotheses depending on labels $k \in \{0,\ldots,9\}$:
\begin{equation}\label{dp:raw_hypo}
    H_0(k): Y \neq k \text{ and } H_1(k): Y = k.
\end{equation}
Hence, the goal of a randomized mechanism $M_{\varepsilon}$ in this case is not to let the classifier $\phi$ correctly classify the true class of a privacy-protected test image $M_{\varepsilon}(D_{\rm test})$.

\vspace{0.1in}
\noindent \textit{Comparison metrics.} We measure the performance of different privacy-protecting randomized mechanisms $M_{\varepsilon}$ at radius $\varepsilon$ by the type-I and type-II errors of the classifier $\phi$ on testing \eqref{dp:raw_hypo} over different classes $k$. Recall the classifier $\phi$ maps an arbitrary input image to a probability distribution over the 10 classes. Given a test dataset $D_{\rm test}=\{X_{\rm test}\}$ with $X_{\rm test}\sim P_{\rm test},$ we let $\hat{Y}(M_{\varepsilon})=\arg\max_{j=0,\ldots,9} \phi(M_{\varepsilon}(D_{\rm test}))_j$ be the predicted class of $M_{\varepsilon}(D_{\rm test})$ by $\phi$. Then, for this particular setting, the type-I error $\alpha(k, M_{\varepsilon})$ and type-II error $\beta(k, M_{\varepsilon})$ are computed as
\begin{align}
    \alpha(k, M_{\varepsilon}) & = \mathbb{P}(\hat{Y}(M_{\varepsilon})=k|Y\neq k) \label{dp:type_I}\\
    \beta(k, M_{\varepsilon}) & = \mathbb{P}(\hat{Y}(M_{\varepsilon})\neq k|Y= k) \label{dp:type_II},
\end{align}
where the probability is taken over test image-label pairs $X_{\rm test}=(X_{\rm test,\rm img},Y_{\rm test})$ for $X_{\rm test} \sim P_{\rm test}$. We then measure the performance of $M_{\varepsilon}$ by taking the average of \eqref{dp:type_I} and \eqref{dp:type_II} over $k$:
\begin{equation}\label{dp:type_I_II_ave}
    \alpha(M_{\varepsilon}) = \sum_{k=0}^9 \alpha(k,M_{\varepsilon})/10, ~\beta(M_{\varepsilon}) = \sum_{k=0}^9 \beta(k,M_{\varepsilon})/10.
\end{equation}
If a mechanism $M_{\varepsilon}$ provides strong privacy, we should expect high values of $\alpha(M_{\varepsilon})$ and $\beta(M_{\varepsilon})$, as the classifier $\phi$ would make high errors on privacy-protected images $M_{\varepsilon}(D_{\rm test})$.

\vspace{0.1in}
\noindent \textit{Results.} Figure \ref{fig:dp_ave_mnist} shows the comparative results by the proposed \flowmodel{} DPM against the APM-G and APM-L baselines. Qualitatively, we notice in (a)-(c) that under the same amount of $\ell_2$ perturbation $\varepsilon$, DPM induces meaningful contextual changes to the queries $q(D)$ (i.e., changing a digit 0 to a digit 8). In contrast, the additive mechanisms only blur the queries slightly. Quantitatively, as shown in (d), such difference helps protect privacy against the decision function $\phi$: the type-I and type-II errors of $\phi$ under our proposed DPM are much higher than those of $\phi$ under the additive perturbation mechanisms. As a result, our DPM is an empirically more effective privacy-protecting mechanism under the same amount of average perturbation as measured in $\varepsilon$.

\subsubsection{MNIST missing digit detection}
\begin{figure}[!t]
    \begin{minipage}{\textwidth}
        \includegraphics[width=\linewidth]{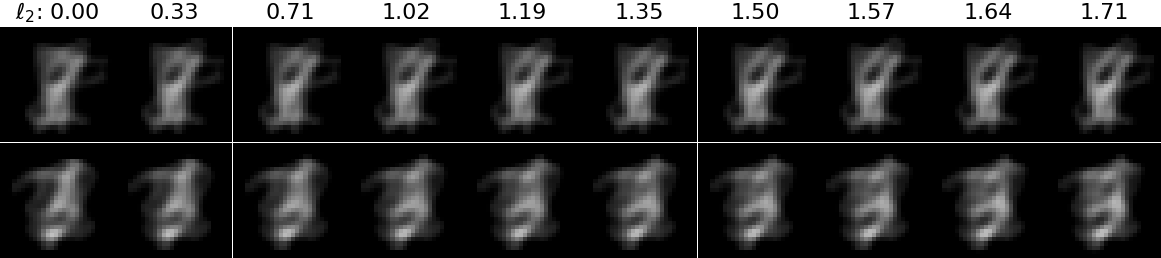}
        \subcaption{DPM by \flowmodel{}}
    \end{minipage}
    \begin{minipage}{\textwidth}
        \includegraphics[width=\linewidth]{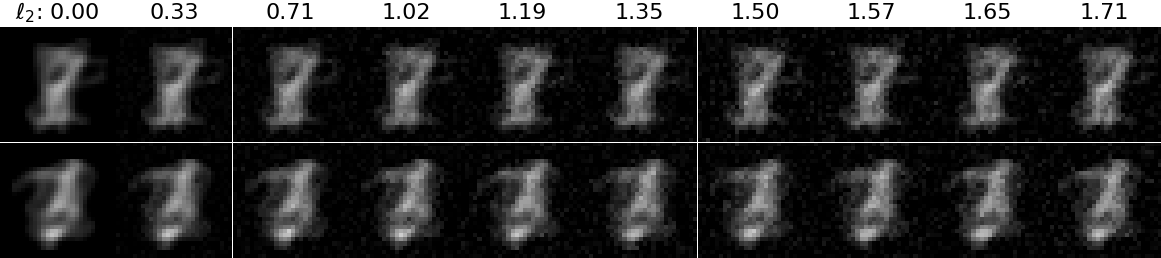}
        \subcaption{APM-G (additive Gaussian)}
    \end{minipage}
    \begin{minipage}{\textwidth}
        \includegraphics[width=\linewidth]{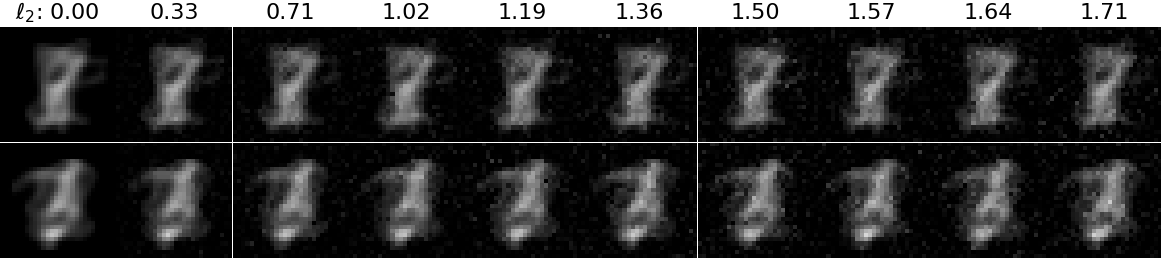}
        \subcaption{APM-L (additive Laplacian)}
    \end{minipage}
    \begin{minipage}{\textwidth}
        \includegraphics[width=\linewidth]{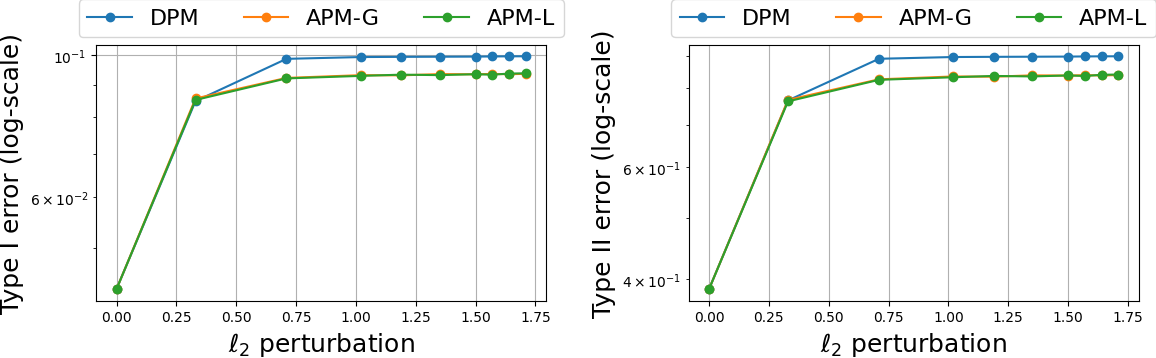}
        \subcaption{Type-I and type-II errors by DPM, APM-G, and APM-L.}
    \end{minipage}
    \caption{Differential privacy example of MNIST missing digit detection. We present similar sets of figures as in Figure \ref{fig:dp_raw_mnist}, where the main difference lies in the definition of dataset $D$ and query function $q(D)$, which returns an average image of images in $D$.}
    \label{fig:dp_ave_mnist}
\end{figure}

We consider an alternative setting that is a type of \textit{membership inference attack} problem \citep{shokri2017membership} and can be viewed as a more natural DP task. In short, we construct \textit{average} images from digits of 9 classes, where the goal of the decision function $\phi$, which is still a 10-class classifier, is to determine the class of the \textit{missing} digit based on a given average image. We follow the notations in section \ref{expr:dp_raw} when describing the setup and metrics. 

\vspace{0.1in}
\noindent \textit{DP setup and comparison metric.} We define (1)--(4) in this new setting. 
For (1), we define a dataset $D=\{X_1,\ldots,X_9: X_i=(X_{{\rm img},i},Y_i) \sim Q, Y_i \neq Y_j \text{ if } i\neq j\}$. Thus, the dataset $D$ has precisely nine random image-label pairs, one from each distinct class. Given two datasets $D$ and $D'$, they are neighbors in the sense that the set of labels $\{Y_i\}$ in $D$ and $D'$ differ by at most one entry. 
For (2), the query function $q(D)=\sum_{i=1}^9 X_{{\rm img}, i}/9$, where the sum is taken pixel-wise so that $q(D)$ returns an average image of the same dimension as raw images. 
For (3), the privacy-protection mechanism $M_{\varepsilon}$ either applies our trained \flowmodel{} model to the average image $q(D)$ or adds random noises to it. 
For (4), given the true missing label $Y(D)$ of the dataset $D$, we then consider the following sets of hypotheses depending on the label $k \in \{0,\ldots,9\}$:
\begin{equation}\label{dp:ave_hypo}
    H_0(k): Y(D) \neq k \text{ and } H_1(k): Y(D) = k.
\end{equation}
In this new setup, we still evaluate the effectiveness of a randomized mechanism $M_{\varepsilon}$ using \eqref{dp:type_I_II_ave}, where the probabilities of type-I and type-II errors are taken over test datasets $D_{\rm test}$, each of which contains nine random test image-label pairs $X_{\rm test}\sim P_{\rm test}$. 

We also explain how we train the classifier $\phi$ and the flow model $T$ in this new setting. The architecture of $\phi$ is still based on convolutional layers, where the training data of $\phi$ consists of $\{q(D), Y(D)\}$, which are the set of raw average images $q(D)$ and corresponding missing labels $Y(D)$. We then train $\phi$ using empirical risk minimization under the cross-entropy loss by sampling mini-batches of datasets $D$. The classifier $\phi$ is thus trained to determine the missing label $Y(D)$ given the average image.
To train the flow model $T$ using Algorithm \ref{alg:block_wise}, we adopt the identical network architecture as in the previous MNIST examples and train three blocks given the classifier $\phi$ with $\gamma_k\equiv 2$.

\vspace{0.1in}
\noindent \textit{Results.} Figure \ref{fig:dp_ave_mnist} shows both qualitative and quantitative comparisons of our proposed DPM against APM-G and APM-L in this more challenging setting. The interpretations of results are similar to those in section \ref{expr:dp_raw}. Specifically, we notice more contextual changes by DPM in subfigure (a) than APMs in subfigures (b) and (c), and the higher type-I and type-II errors in subfigure (d) demonstrate the benefit of DPM at protecting privacy against a pre-trained decision function $\phi$.

\section{Summary and Discussion}\label{sec:discussion}

In this paper, we have presented a computational framework called \flowmodel{} for solving the worst-case distribution, the Least Favorable Distributions (LFD), in Wasserstein Distributionally Robust Optimization (WDRO). Specifically, the worst-case distribution is found as the push-forward distribution induced by our \flowmodel{} model on the data distribution, and the entire probability trajectory is continuous and invertible due to the use of flow models. We demonstrate the utility of \flowmodel{} in various applications of DRO, including adversarial attacks of pre-trained image classifiers and differential privacy protection through our distributional perturbation mechanism. \flowmodel{} demonstrates strong improvement against baseline methods on high-dimensional data.

There are a few future directions to extend the work. Here, we set aside the min-max exchange issue for the following reasons. It has been shown in the original contribution \cite{mohajerin2018data} that when the reference measure (i.e., the center of the uncertainty set) is empirical distribution and thus discrete, the problem \eqref{eq:DRO_general} has {\it strong duality}: one can exchange the min and max in the formulation and the solutions for the primal and the dual problems are the same when the loss function is convex-concave in the vector space. The results are shown leveraging the fact that the worst-case distributions for the Wasserstein DRO problem are discrete when the reference measure is discrete, thus reducing the infinite-dimensional optimization problem to a finite-dimensional minimax problem. Thus, one can invoke the standard minimax theorem (see, e.g., \cite{ben2001lectures}). Here, since later on we restrict the LFD to be a continuous function, the strong duality proof in \cite{mohajerin2018data} no longer carries through, and one has to extend the minimax theorem (e.g., \cite{rockafellar1997convex} and \cite{ben2001lectures} using Kakutani theorem) for the most general version involving functionals that are geodesic convex on the manifold of distribution functions; the proof is rather technical, and we leave it for future work. 
Second, theoretically, how to formalize our distributional perturbation mechanism on high-dimensional queries to make it satisfy a DP criterion is also an important question. 
Lastly, our approach is general and does not rely on neural networks. In future work, one can potentially extend to other alternative representations of the optimal transport maps that work particularly well for low-dimensional and small sample settings.

\section*{Acknowledgement}

This work is partially supported by an NSF CAREER CCF-1650913, NSF DMS-2134037, CMMI-2015787, CMMI-2112533, DMS-1938106, DMS-1830210, and the Coca-Cola Foundation. 
XC is also partially supported by 
NSF DMS-2237842
and Simons Foundation.
The authors would like to thank the helpful discussion with Dr. Daniel Kuhn, Dr. Jose Blanchet, Dr. Arkadi Nemirovski, Dr. Alexander Shapiro, and Dr. Georgia-Ann Klutke.

\bibliographystyle{IEEEtran} 
\bibliography{references}

\appendix

\section{Proofs}\label{app:proof}

\noindent 
$\bullet$ Proofs in Section \ref{sec:framework}

\begin{proof}[Proof of Lemma \ref{lemma:finite-W2}]
    By  the definition \eqref{eq:deff-W2-Kantorovich},  let $\pi = \mu \times \nu$ be a coupling of $\mu$ and $\nu$,
\begin{align*}
      \W_2^2( \mu, \nu )  
   &   \le   \int_{ \R^d \times \R^d} \| x-y \|^2 d\pi(x,y) \\
   &  \le   \int_{ \R^d \times \R^d} 2( \| x\| ^2  + \|y \|^2) d\pi(x,y) \\
   &  = 2 (M_2(\mu) + M_2(\nu)  ) <\infty. 
\end{align*}
\end{proof}

\begin{proof}[Proof of Proposition \ref{prop:W2-PR-by-T}]
Since $P \in \calP_2^r(\calX)$, $\forall \mu \in \calP_2(\calX)$, the OT map from $P$ to $\mu$ uniquely exists due to Brenier Theorem \citep{brenier1991polar}
and we denote by $T_P^\mu$. We use the notation $T_P^\cdot$ for $\cdot $ in $\calP_2(\calX)$ throughout the proof.

For any $\mu \in \calP_2$, define $M_2( \mu ):= \int |x|^2 d\mu$.
Observe that  $M_2( T_\# P) = \E_{x \sim P} \| T(x) \|^2 $, and this means
\begin{equation}\label{eq:T2-is-L2}
T \in L^2(P) \quad \text{iff.} \quad T_\# P \in \calP_2.
\end{equation}
We first verify that $L_\mu: \calP_2 \to (-\infty, +\infty)$,
 and $L_T: L^2(P) \to (-\infty, +\infty)$.
 Then, whenever the minimum is attained,  it is a finite value for both problems.
For $L_\mu$, $\forall \mu \in \calP_2(\calX)$, $\varphi( \mu) $ is finite, and $ 0 \le W_2^2( P,\mu)  < \infty$.
(Because $W_2^2( P,\mu) = \E_{x\sim P} \| x - T_P^\mu (x)\|^2 \le 2( M_2(P) + M_2(\mu)) < \infty$.)
 Thus $L_\mu( \mu) $ is finite.
For $L_T$, $\forall T \in L^2(P)$, 
by \eqref{eq:T2-is-L2}, $T_\# P \in \calP_2$, then
$0 \le \E_{x \sim P} \| x - T(x)\|^2 \le 2 ( M_2(P) + \E_{x\sim P} \| T(x)\|^2 )  < \infty $.  %
Also, $\varphi( T_\# P )$ is finite, and then $L_T [T]  $ is finite.

\underline{To show (a)}: Let $Q:= (T^*)_\#P$, since $T^* \in L^2(P)$, by \eqref{eq:T2-is-L2}, 
$Q  \in \calP_2(\calX)$. We have
\begin{equation}\label{eq:pf-a-1}
L_T^* = L_T( T^*) = \varphi( Q ) + \lambda \E_{x \sim P}  \| x-T^*(x) \|^{2}.
\end{equation}
We claim that $\E_{x \sim P}  \| x-T^*(x) \|^{2} = \W_2^2( P, Q)$:
the ``$\ge$'' always holds; 
To show that it cannot be ``$>$'', 
note that $ \W_2^2( P, Q) = \E_{x \sim P}  \| x-T_P^Q(x) \|^{2}$,
and since $Q \in \calP_2$, 
 $T_P^Q \in L^2(P)$ by \eqref{eq:T2-is-L2}.
If $\E_{x \sim P}  \| x-T^*(x) \|^{2} > \W_2^2( P, Q)$, then
\[
L_T (T_P^Q) = \varphi( Q) + \lambda \W_2^2( P, Q) 
< \varphi( Q) + \lambda \E_{x \sim P}  \| x-T^*(x) \|^{2} = L_T( T^*),
\]
contradicting with that $T^*$ attains the minimum of $L_T$. 

As a result, back to \eqref{eq:pf-a-1}, we have
\begin{equation}\label{eq:pf-a-2}
L_T^* =  \varphi( Q ) + \lambda \W_2^2( P, Q) = L_\mu (Q).
\end{equation}
It remains to show that $Q$ is a minimizer of \eqref{eq:def-PR-mu-lemma}, and then 
$L_\mu(Q) = L_\mu^* = L_T^*$.
To verify this, $\forall \mu \in \calP_2(\calX)$,
\begin{equation}\label{eq:pf-Lmu-LT-OT}
L_\mu ( \mu ) 
= \varphi( \mu ) + \lambda \W_2^2( P, \mu )
=  \varphi( \mu ) + \lambda \E_{x \sim P} \| x - T_P^\mu (x) \|^2 
= L_T (T_P^\mu),
\end{equation}
where $T_P^\mu \in L^2(P)$ (since $\mu \in \calP_2$), 
and thus $L_T (T_P^\mu) \ge L_T^*$. Combined with \eqref{eq:pf-a-2}, 
$L_\mu (\mu) \ge  L_T^* = L_\mu(Q)$.  This shows that $Q$ is a minimizer of $L_\mu$ on $\calP_2(\calX)$, and finishes the proof of (a).

\underline{To show (b)}: Since \eqref{eq:pf-Lmu-LT-OT} holds for $\mu^* \in \calP_2(\calX)$, we have
\begin{equation}\label{eq:pf-b-1}
L_\mu^* = L_\mu (\mu^*) 
= L_T ( T_P^{\mu^*} ),
\end{equation}
and $T_P^{\mu^*} \in L^2(P)$. 
It remains to show that $T_P^{\mu^*}$ is a minimizer of \eqref{eq:def-PR-T-lemma}, and then $L_T^* = L_T ( T_P^{\mu^*} ) = L_\mu^*$.
To show this, $\forall T \in  L^2(P)$, $T_\# P \in \calP_2$, and
\[
L_T( T) 
= \varphi( T_\#P ) + \lambda \E_{x \sim P} \| x - T (x) \|^2 
\ge \varphi( T_\#P ) + \lambda \W_2^2( P, T_\#P) = L_\mu (T_\#P ), 
\]
and meanwhile $L_\mu (T_\#P )  \ge L_\mu^*$.
We thus have $L_T(T) \ge L_\mu^*$.
Together with \eqref{eq:pf-b-1}, this gives that $  L_T( T)  \ge  L_T ( T_P^{\mu^*} )$, which means that $T_P^{\mu^*}$ is a minimizer of $L_T$ on $L^2(P)$.
\end{proof}

\begin{proof}[Proof of Proposition \ref{prop:dual-discrete}]
Because for each $x_i$, there is $z_i \in \R^d $ which minimizes $\inf_{z}  \left[ V(z) + \lambda \| x_i - z\|^2 \right]$,
we define $z_i = T(x_i)$,
and let $\hat Q := T_\# \hat P$, namely, $\hat Q$ is the empirical distribution for $\{ z_i\}_{i=1}^n$. By definition, 
\begin{align*}
\E_{x \sim \hat P} \inf_{z}  \left[ V(z) + \lambda \| x - z\|^2 \right]
& =
\frac{1}{n}\sum_{i=1}^n 
\left[ V(z_i) + \lambda \| x_i - z_i\|^2 \right] \\
& =
\E_{ x\sim \hat Q} V( x ) + 
 \lambda \E_{ x\sim \hat P} \| x - T(x)\|^2. 
\end{align*}
Because $({\rm I_d}, T)_\# \hat P$ is a coupling of $\hat P$ and $\hat Q$, we have
\[
\E_{x \sim \hat P} \| x - T(x)\|^2 \ge \W_2^2(\hat P, \hat Q),
\]
and then we have
\[
\E_{x \sim \hat P} \inf_{z}  \left[ V(z) + \lambda \| x - z\|^2 \right]
\ge \E_{ x\sim \hat Q} V( x ) + 
 \lambda \W_2^2(\hat P, \hat Q).
\]
Meanwhile, $Q \in \calP_2$ because $\{ z_i\}_{i=1}^n$ is a finite set in $\R^d$. 
Thus, $\E_{ x \sim \hat Q} V(x) + \lambda \W_2^2( \hat P, \hat Q)$ is greater than or equal to the l.h.s. of \eqref{eq:dual-form-hatP}.
This proves that 
\[
 \E_{x \sim \hat P} \inf_{z}  \left[ V(z) + \lambda \| x - z\|^2 \right]
 \ge \min_{Q \in \calP_2} \E_{x \sim Q} V(x) + \lambda  \W_2^2( \hat P, Q).
\]

To prove the other direction, we consider any $Q \in \calP_2$, 
\begin{align}
\E_{x \sim Q} V(x) + \lambda  \W_2^2( \hat P, Q)
& = \E_{z \sim Q} V(z)
+ \lambda \inf_{\pi \in \prod (\hat P, Q)} \E_{(x,z )\sim \pi} \| x - z\|^2  \nonumber \\
& = \inf_{\pi \in \prod (\hat P, Q)}
\E_{(x,z )\sim \pi} \left[  V(z) +\lambda \| x - z\|^2 \right]  \nonumber \\
& \ge \inf_{\pi \in \prod (\hat P, Q)}
\E_{(x,z )\sim \pi}  \inf_{z'}  \left[ V(z') + \lambda \| x - z'\|^2 \right] \quad \text{(by that $\inf_{z'}$ is attained)} \nonumber \\
& = \E_{x\sim \hat P}  \inf_{z'}  \left[ V(z') + \lambda \| x - z'\|^2 \right]. \label{eq:pf-gt-explicit-form}
\end{align}
Because $\E_{ x \sim \hat P}$ involves a finite summation, the r.h.s. of \eqref{eq:pf-gt-explicit-form} is a finite number.
Since \eqref{eq:pf-gt-explicit-form} for any $Q \in \calP_2$, we have
\[
 \min_{Q \in \calP_2} \E_{x \sim Q} V(x) + \lambda  \W_2^2( \hat P, Q)
 \ge \E_{x\sim \hat P}  \inf_{z'}  \left[ V(z') + \lambda \| x - z'\|^2 \right].
\]
\end{proof}

\vspace{10pt}
\noindent 
$\bullet$ Proofs in Section \ref{sec:theory}

\begin{proof}[Proof of Lemma \ref{lemma:L2-perturb}]
We are to verify that $M_2( T_\# \mu ) < \infty$. 
\[
\int_{\R^d} |x|^2 d(T_\# \mu)(x)
= \int_{\R^d} \|T(x)\|^2 d\mu(x)
= \| T\|_\mu^2 <\infty,
\]
since $T \in L^2(\mu)$.
\end{proof}

\begin{proof}[Proof of Lemma \ref{lemma:diff-phi}]
We first verify that $\forall \mu \in \calP_2$, $\varphi(\mu) $ is finite. 
Under Assumption \ref{assump:V}, one can verify %
that
\begin{equation}\label{eq:cond-V}
| V(y) - V(x) - \nabla V(x)^T(y-x) |
\le  \rev{ \frac{L}{2} }\| y-x\|^2,\quad \forall x, y \in \R^d.  
\end{equation}
By taking $x=0$, \eqref{eq:cond-V} implies that, 
\begin{equation}
 V(y) 
\le  V(0) + \nabla V(0)^Ty + \rev{ \frac{L}{2} } \| y\|^2,
\quad \forall y \in \R^d, 
\end{equation}
and thus there are finite numbers $a , b$ s.t. 
\[
V(y) 
\le a + b \| y\|^2, \quad \forall y \in \R^d.
\]
Since $M_2(\mu) < \infty$, $\varphi (\mu) = \int V d\mu < \infty$. 
Similarly, \eqref{eq:cond-V} also implies
\begin{equation}
- V(y)  
\le -  V(0) - \nabla V(0)^Ty + \rev{ \frac{L}{2} }  \| y\|^2,
\quad \forall y \in \R^d, 
\end{equation}
and thus $- V(y)$ is also upper-bounded by $a + b \| y\|^2$ 
for some finite $a, b $. Thus $-\varphi(\mu) = \int (-V)d\mu <\infty $. 
This proves that $\varphi: \calP_2 \to (-\infty, \infty)$, and thus we will not add/subtract $\pm \infty$ in our derivations below. 

Next, we verify that for any $\mu \in \calP_2$, 
$\nabla V \in L^2(\mu)$,
and then  the inner-product $\langle \nabla V, v \rangle_\mu$ in \eqref{eq:strong-diff-phi} is  well-defined.
To see this is the case, by that $\nabla V$ is $L$-Lipschitz, we have
\[
\| \nabla V(x) \| \le \| \nabla V(0) \| + L \| x \|, \quad \forall x \in \R^d,
\]
and then, let $c = \| \nabla V(0) \| $, 
\[
\| \nabla V \|_\mu^2 
=  \int_{\R^d} \| \nabla V(x) \|^2 d\mu(x)
\le \int_{\R^d} 2( c^2 + L^2 \| x\|^2) d\mu(x) 
= 2 (c^2 + L^2 M_2(\mu) ) < \infty.
\]

We are ready to prove  \eqref{eq:strong-diff-phi}.
For any $\mu \in \calP_2$ and $v \in L_2(\mu)$, 
\[
\varphi( ( {\rm I_d} + \delta v)_\# \mu)
= \int_{\R^d} V(x + \delta v(x))  d\mu(x),
\]
and thus
\begin{align}
& \varphi( ( {\rm I_d} + \delta v)_\# \mu)
 - \varphi( \mu) 
 - \delta \langle \nabla V, v \rangle_\mu  \nonumber \\
= &  \int_{\R^d} 
    \left( V( x + \delta v(x))  - V(x) 
    -   \delta\nabla V(x)^T v(x)  \right) d\mu(x).  \label{eq:pf-diff-phi-1}
\end{align}
Combined with \eqref{eq:cond-V} where $y = x + \delta v(x)$, 
the absolute of r.h.s. of \eqref{eq:pf-diff-phi-1} is upper-bounded by
\begin{align}
|\eqref{eq:pf-diff-phi-1}|
& \le  \int_{\R^d} 
    \left| V( x + \delta v(x))  - V(x) 
    -   \delta\nabla V(x)^T v(x)  \right| d\mu(x) \nonumber \\
& \le  \rev{ \frac{L}{2} }   \delta^2 \int_{\R^d} 
     \|v(x)\|^2 d\mu(x) 
     = \rev{ \frac{L}{2} }  \delta^2,
\end{align}
where in the last equality we used that $\| v \|_\mu =1$.
This shows that 
\[
 \varphi( ( {\rm I_d} + \delta v)_\# \mu)
 - \varphi( \mu) 
 - \delta \langle \nabla V, v \rangle_\mu 
 = O(\delta^2)
\]
which proves \eqref{eq:strong-diff-phi}.
\end{proof}

\begin{proof}[Proof of Lemma \ref{lemma:superdiff-psi}]
We first verify that the quantities 
$\psi( \mu) $, $ \psi( ( {\rm I_d} + \delta v)_\# \mu) $ are all finite and $ {\rm I_d} - T_\mu^P \in L_2(\mu)$.
$\psi( \mu) = \frac{1}{2}\W_2^2( \mu , P)$ is finite because  $\mu, P \in \calP_2$ and the Lemma \ref{lemma:finite-W2} applies. 
For the same reason, $ \psi( ( {\rm I_d} + \delta v)_\# \mu) $ will be finite if $( {\rm I_d} + \delta v)_\# \mu \in \calP_2$, 
which always holds by  Lemma \ref{lemma:L2-perturb} and that $ v \in L^2(\mu)$.
To verify that $ {\rm I_d} - T_\mu^P \in L_2(\mu)$,
use the fact that 
$\W_2^2( \mu, P)  = \E_{ x \sim \mu} \| x - T_\mu^P(x)\|^2 = \| {\rm I_d} - T_\mu^P \|_\mu^2$.

To prove \eqref{eq:strong-superdiff-psi}, 
define $\tilde \mu := ( {\rm I_d} + \delta v)_\# \mu$,
and  observe that 
$ ( {\rm I_d} + \delta v, T_\mu^P)_\# \mu $ is a coupling of  $\tilde \mu$ and $P$, then
\[
\W_2^2( \tilde \mu, P)
\le \E_{x \sim \mu} \| x + \delta v(x) - T_\mu^P (x)  \|^2
= \|  {\rm I_d} + \delta v -  T_\mu^P\|_\mu^2.
\]
Expanding the r.h.s, we have
\[
\|  ( {\rm I_d} -  T_\mu^P ) + \delta v  \|_\mu^2
= \| {\rm I_d} -  T_\mu^P \|_\mu^2 + 2 \langle {\rm I_d} -  T_\mu^P , \delta v \rangle_\mu + \| \delta v  \|_\mu^2,
\]
and note that 
\[
\| {\rm I_d} -  T_\mu^P \|_\mu^2 = \W_2^2 ( \mu, P ).
\]
Putting together, we have that (recall $ \| v  \|_\mu=1$)
\[
\W_2^2( \tilde \mu, P)
\le  \W_2^2 ( \mu, P )
 + 2  \delta \langle {\rm I_d} -  T_\mu^P ,  v \rangle_\mu + \delta^2,
 \]
which gives that 
\[
\psi( \tilde \mu  ) \le \psi( \mu) +  \delta \langle {\rm I_d} -  T_\mu^P ,  v \rangle_\mu +  \delta^2/2.
\]
This proves \eqref{eq:strong-superdiff-psi}.
\end{proof}

\begin{proof}[Proof of Theorem \ref{prop:TR}]
We consider the two cases respectively.

\underline{To show (i)}: 
Since $\W_2(Q, P) < \varepsilon$, there exists a small neighborhood $\calB(Q)$ of $Q$ under $\W_2$ distance
satisfying that $\calB(Q) \subset \calB_\varepsilon$. 
For any unit vector $v \in L^2(Q)$ and $\delta >0$, define
\begin{equation}\label{eq:def-deltaQ-del}
\tilde Q_\delta := ( {\rm I_d} + \delta v)_\# Q,
\end{equation}
and observe that 
$\W_2^2( \tilde Q_\delta, Q) \le \| \delta v\|_Q^2 = \delta^2$.
Thus, $\exists \delta(v) > 0$, s.t. 
$\tilde Q_\delta \in \calB(Q)$ whenever $\delta < \delta(v)$. 
Then by that $Q$ is a local minimum of \eqref{eq:problem-W2-trust-region},
\begin{equation}\label{eq:pf-(i)-minimum}
\varphi( \tilde Q_\delta ) \ge \varphi( Q), \quad \forall \delta < \delta(v).
\end{equation}
Meanwhile, by Lemma \ref{lemma:diff-phi}, $\nabla V \in L^2(Q)$, and 
\begin{equation}\label{eq:pf-tr-perturb-phi}
    \varphi(\tilde Q_\delta  ) = \varphi( Q) 
    + \delta \langle \nabla V, v \rangle_Q + o (  \delta ).
\end{equation}
Then \eqref{eq:pf-(i)-minimum} can hold only if
\[
\langle \nabla V, v \rangle_Q \ge 0.
\]
We have derived that 
\[
\langle \nabla V, v \rangle_Q \ge 0, \quad \forall  v \in L^2(Q), \, \|v\|_Q =1,
\]
and this means that $\nabla V = 0 $ in $L^2(Q)$, which proves (i).

\underline{To show (ii)}: 
Note that $\W_2^2(  Q, P ) = \varepsilon^2$, and $T_Q^P$ is defined $Q$-a.e.
For any unit vector $v \in L^2(Q)$ and $\delta > 0$, define $\tilde Q_\delta$ as in \eqref{eq:def-deltaQ-del}.
By Lemma \ref{lemma:superdiff-psi},
\begin{equation}\label{eq:pf-tr-perturb-psi}
 \W_2^2( \tilde Q_\delta, P ) \le 
  \W_2^2(  Q, P )
	+ 2 \delta \langle  {\rm I_d} - T_Q^P , v \rangle_Q + o (  \delta ),   
\end{equation}
thus, if $ \langle  {\rm I_d} - T_Q^P , v \rangle_\mu < 0$, then for small enough $\delta$ we can make $ \W_2^2( \tilde Q_\delta, P ) < 
  \W_2^2(  Q, P )$. That is, $\exists \delta(v) > 0$ s.t.
  \[
   \W_2( \tilde Q_\delta, P ) <  \W_2 (  Q, P ) = \varepsilon, \quad \forall \delta < \delta(v).
  \]
Now, by that $Q$ is a local minimum of  \eqref{eq:problem-W2-trust-region}, we again have \eqref{eq:pf-(i)-minimum} hold.
Similar as in (i), we must have 
\[
\langle \nabla V, v \rangle_Q \ge 0.
\]
Thus we have derived that 
\begin{equation}\label{eq:pf-(ii)-for-any-v}
\langle \nabla V, v \rangle_Q \ge 0,
\quad \forall v \in L^2(Q), \, \|v\|_Q =1, \text{ and } \langle  {\rm I_d} - T_Q^P , v \rangle_\mu < 0.
\end{equation}
Then, either $\nabla V = 0 $ which renders $\langle \nabla V, v \rangle_Q \equiv 0$;
Alternatively, 
if $\nabla V $ is a non-zero vector in $L^2(Q)$, 
since $ {\rm I_d} - T_Q^P$ is also a non-zero vector in $L^2(Q)$
by that 
\[
\| {\rm I_d} - T_Q^P \|_Q^2 = \W_2^2 (Q, P)^2 = \varepsilon^2 > 0,
\]
we can have \eqref{eq:pf-(ii)-for-any-v} hold only if the two vectors 
 $ - \nabla V $ and $ {\rm I_d} - T_Q^P$ are parallel and aligned in direction in $L^2(Q)$.
 This means that $\exists \lambda > 0$ s.t.
 $- \nabla V  = \lambda(  {\rm I_d} - T_Q^P )$ in $L^2(Q)$,
 which proves (ii).
\end{proof}

\begin{proof}[Proof of Theorem \ref{prop:PR}]
For fixed $\gamma > 0$, define 
\[
F_\gamma (\mu) := \varphi( \mu )
    	+  \frac{1}{2\gamma } \calW_2^2( \mu, P)
     =\varphi( \mu ) + \frac{1}{\gamma} \psi( \mu) .
\]
For any $\mu \in \calP_2$, by Lemma \ref{lemma:diff-phi}, $\varphi(\mu) $ is always finite, and $ 0 \le \W_2^2(\mu ,P) < \infty$ by Lemma \ref{lemma:finite-W2}. 
Thus $F_\gamma : \calP_2 \to (-\infty, +\infty)$. 

By that  $Q$ is a local minimum of $F_\gamma $ in $\calP_2$, we have that for any unit vector $v \in L^2(Q)$ and $\delta > 0$, define $\tilde Q_\delta$ as in \eqref{eq:def-deltaQ-del}, then $\exists \delta (v) > 0 $ s.t,
\begin{equation}\label{eq:pf-PR-minimum}
F_\gamma ( \tilde Q_\delta ) \ge F_\gamma (Q),\quad \forall \delta < \delta(v).    
\end{equation}
Same as in the proof of Theorem \ref{prop:TR},
by Lemma \ref{lemma:diff-phi}, $\nabla V \in L^2(Q)$, and 
\eqref{eq:pf-tr-perturb-phi} holds;
Because $Q \in \calP_2^r$,  by Lemma \ref{lemma:superdiff-psi}, 
\eqref{eq:pf-tr-perturb-psi} holds, which gives
\begin{equation}
 \psi( \tilde Q_\delta) \le 
  \psi(  Q )
+  \delta \langle  {\rm I_d} - T_Q^P , v \rangle_\mu + o (  \delta ).  
\end{equation}
Putting together, we have
\[
 F_\gamma ( \tilde Q_\delta)
 \le  F_\gamma (  Q)
 + \delta \langle \nabla V + \frac{1}{\gamma} (  {\rm I_d} - T_Q^P ), v \rangle_Q 
 + o (  \delta ),
\]
and then for \eqref{eq:pf-PR-minimum} to hold, we must have 
\[
\langle \nabla V + \frac{1}{\gamma} (  {\rm I_d} - T_Q^P ), v \rangle_Q  \ge 0.
\]
We have derived that 
\[
\langle \nabla V + \frac{1}{\gamma} (  {\rm I_d} - T_Q^P ), v \rangle_Q  \ge 0,
\quad \forall  v \in L^2(Q), \, \|v\|_Q =1,
\]
and this means that $ \nabla V + \frac{1}{\gamma} (  {\rm I_d} - T_Q^P ) = 0 $ in $L^2(Q)$, which proves \eqref{eq:1st-order-PR}. 
\end{proof}

\begin{proof}[Proof of Corollary \ref{cor:dual-form}]
    We first verify that  for every $x \in \R^d$,  the minimization 
    \[
    \inf_{ z}
    F(z):=
     V(z) + \frac{1}{2\gamma } \| z - x\|^2
    \]
    has unique minimizer $z^*$ s.t.
    \begin{equation}\label{eq:moreau-1st-order}
   \gamma  \nabla V(z^*) +  (z^* - x) = 0 ,
    \end{equation}
    and then $u(x,\gamma)$ is well-defined. 
    Under Assumption \ref{assump:V},
    we have \eqref{eq:cond-V}.
    Together with that  \rev{$\frac{1}{2\gamma} > \frac{L}{2}$}, we have that $F $ is strongly convex on $\R^d$, and then the above holds. 

It remains to 

(1) Show that  $Q$ is a global minimum of \eqref{eq:problem-W2-proximal-GD}, 

(2) Prove \eqref{eq:dual-LFD-g-gamma}, which can be equivalently written as
  \begin{equation*}
  \min_{Q' \in \calP_2} \E_{x \sim Q'} V(x) + \frac{1}{2\gamma} \W_2^2( P, Q')
  =  \E_{x \sim P} u( x, \gamma)
  =   \E_{x \sim P} \inf_z \left[  V(z) + \frac{1}{2\gamma } \| z - x\|^2 \right].
  \end{equation*}

We proceed by observing that 
  \begin{equation}\label{eq:pf-dual-form-2}
  \min_{Q' \in \calP_2} \E_{x \sim Q'} V(x) + \frac{1}{2\gamma} \W_2^2( P, Q')
  \ge \E_{x \sim P} \inf_z \left[  V(z) + \frac{1}{2\gamma } \| z - x\|^2 \right],
  \end{equation}
  and to verify that at the local minimum $Q$, 
  \begin{equation}\label{eq:pf-dual-form-3}
   \E_{x \sim Q} V(x) + \frac{1}{2\gamma} \W_2^2( P, Q)
   = \E_{x \sim P} \inf_z \left[  V(z) + \frac{1}{2\gamma } \| z - x\|^2 \right].
   \end{equation}
     
If both \eqref{eq:pf-dual-form-2} and \eqref{eq:pf-dual-form-3} hold, then we have (1) and (2)
and the corollary is proved.
   Note that $0 \le \W_2(P, Q) <\infty$ (by Lemma \ref{lemma:finite-W2}), 
   and  $\E_{x \sim Q} V(x) = \varphi( Q) $ is finite (by Lemma \ref{lemma:diff-phi}), 
then   the l.h.s. of \eqref{eq:pf-dual-form-3} is finite.  
Thus  \eqref{eq:pf-dual-form-3} would imply that the global minimum of \eqref{eq:problem-W2-proximal-GD},
 which also takes the form as the r.h.s. of  \eqref{eq:pf-dual-form-3}, is a finite number.

\vspace{5pt}
\underline{Proof of \eqref{eq:pf-dual-form-2}}:
This is a similar argument as in the proof of Proposition \ref{prop:dual-discrete}.
For any $Q' \in \calP_2$, 
by using the same derivation as in \eqref{eq:pf-gt-explicit-form} where we replace $\hat P$ to be $P$  (which does not affect the argument)
and $\lambda$ to be $\frac{1}{2\gamma}$, noting that $\inf_z F(z)$ is always attained for any $x$, 
we have 
\[
\E_{x \sim Q'} V(x) + \frac{1}{2\gamma }  \W_2^2(  P, Q')
\ge 
\E_{x \sim P}  \inf_{z}  \left[ V(z) + \frac{1}{2\gamma } \| x - z\|^2 \right].
\]
Since this inequality holds for any $Q' \in \calP_2$,
we have \eqref{eq:pf-dual-form-2}.

\vspace{5pt}
\underline{Proof of \eqref{eq:pf-dual-form-3}}:
By Theorem \ref{prop:PR} which holds under the assumption, we have \eqref{eq:1st-order-PR}, which means that  for $z \in \R^d$ except for a $Q$-measure zero set, 
\[
\gamma \nabla V(z) + ( z - T_Q^P(z) ) = 0.
\]
Consider $u(x,\gamma) $  at $x = T_Q^P(z)$, the unique minimizer of $\inf_z F(z)$  is characterized by the first order condition \eqref{eq:moreau-1st-order}, which $z$ satisfies. This means that  $z$ is the unique minimizer, namely
\begin{equation}\label{eq:pf-Fz-equal-u}
F(z)
= \inf_{z'} \left[  V(z') + \frac{1}{2\gamma } \| z' -  T_Q^P(z) \|^2 \right] 
=  u(  T_Q^P(z) , \gamma).
\end{equation}
Since \eqref{eq:pf-Fz-equal-u} holds for $Q$-a.e. $z$, we have
\begin{equation}\label{eq:pf-dual-form-4}
\E_{z \sim Q}F(z) = \E_{z \sim Q} u(  T_Q^P(z) , \gamma).
\end{equation}
By definition, the l.h.s. of \eqref{eq:pf-dual-form-4} equals
\[
 \E_{z \sim Q}  \left[ V(z) + \frac{1}{2\gamma } \| z -  T_Q^P(z) \|^2 \right]
 =  \E_{z \sim Q} V(z) +  \frac{1}{2\gamma } \W_2^2( P, Q),
\]
which is the l.f.s. of \eqref{eq:pf-dual-form-3}.
The r.h.s. of \eqref{eq:pf-dual-form-4}, by that $(T_Q^P)_\# Q = P$, equals  $\E_{x \sim P} u( x, \gamma)$ and that is 
the r.h.s. of \eqref{eq:pf-dual-form-3}.
Then \eqref{eq:pf-dual-form-4} implies the equality  \eqref{eq:pf-dual-form-3}.
\end{proof}

\vspace{10pt}
\noindent 
$\bullet$ Proof in Section \ref{sec:applications}

\begin{proof}[Proof of Proposition \ref{prop:R-dist}]
   By definition of $\Tind$, we have $\bE_{x\sim P} \|\Tind(x)-x\|_2^2 \leq \varepsilon^2$. 
   Because $(\Tind , \, {\rm I_d})_\# P$ is a coupling of $\QstarPoint$ and $P$, we have that $W_2^2(\QstarPoint, P)\leq \varepsilon^2$ and therefore, $\QstarPoint\in \calB_\varepsilon(P)$ for $\calB_\varepsilon(P)$ defined in \eqref{eq:W2_uncertainty_set}. 
    As a result, we have $\cR(\QstarDist, \phi)=\max_{Q\in \calB_\varepsilon(P)} \cR(Q,\phi) \geq \cR(\QstarPoint,\phi).$
\end{proof}

\section{Experimental details}
\subsection{\rev{Comparison with existing DRO methods.}} \label{app:dro_compare}

We first describe the setup details of finding the LFD and training a robust classifier via solving the DRO problem. We also explain how the WDRO method is formulated at the end for completeness.

\paragraph{Setup of finding the LFD} 
We first train $\phi$ as a two-layer CNN image classifier on the raw MNIST digits via minimizing the cross-entropy loss, where the trained classifier reaches 99\% test accuracy. 
As a data-preprocessing step, before finding the LFD, we train an auto-encoder consisting of convolutional layers in the encoder and convolutional transpose layers in the decoder. 
We then encode raw MNIST digits of dimension 784 as latent codes of dimension 16, where these latent codes are subsequently used as training data for WDRO and \flowmodel{} to find LFD in the latent space. 
Finally, to evaluate the original CNN classifier $\phi$ on the LFD samples, we decode these latent LFD samples back to the image space of dimension 784 and apply $\phi$ on these decoded samples.
Regarding the hyper-parameter of \flowmodel{}, we train three flow blocks with $\gamma_k \equiv 5$. The velocity field $f(x(t),t;\theta)$ is parameterized as a three-layer fully-connected neural network with a hidden dimension equal to 256.

\paragraph{Setup of training a robust classifier on MNIST digits.} 
Our setup resembles training a robust classifier on MNIST digits in \citep{sinha2018certifiable}. The goal is to train robust classifiers to defend against unseen adversarial attacks on the test MNIST images. Specifically, the attack method is chosen as PGD \citep{madry2018towards} under $\ell_2$ and $\ell_{\infty}$ norms, where we then evaluate trained robust classifiers on images attacked with increasing amount of attack budget. If a classifier can reach lower classification errors on the same set of attacked images than another, we consider such a classifier to be more robust. 

Regarding inputs to our proposed Algorithm \ref{alg:min_max}, the training data are raw MNIST digits of dimension 784, the regularization parameter $\gamma=5$, total iteration $N=2500$ (binary digits) or $N=10000$ (full digits), and the number of inner loops $N_{\rm inner}=3$. The image classifier is a three-layer CNN, and we design the flow model to be a ResNet block $T(x;\theta)=x+f(x;\theta)$. Specifically, the architecture of $f(x;\theta)$ is a CNN auto-encoder with an encoder having channels $1-128-256-512$, kernel sizes $8,6,5$, strides $2,2,1$, and padding $3,0,0$. The decoder has channels $512-256-128-1$, kernel sizes $5,6,8$, strides $1,2,2$, and padding $0,0,3$. We also find it effective to train $K$ separate ResNet blocks for $K$ digit classes, so we trained 2 ResBlocks for binary MNIST digits and 10 ResBlocks for fill MNIST digits. 
In addition, for a fair comparison, we use the same regularization parameter $\gamma=5$ in WRM and appropriately adjust its inner loop of finding the sample-wise LFD to match the training time as our proposed method on a single A100 GPU.

\paragraph{Setup of training a robust classifier on CIFAR10 images.} The overall setup is identical to training a robust classifier on MNIST digits, as described above. The differences are: 
\begin{itemize}
    \item Rather than using a two-layer CNN image classifier, we use a ResNet-18 \citep{he2016deep} classifier, which is pre-trained on the original ImageNet \citep{russakovsky2015imagenet} dataset. 
    \item Instead of using WRM and our FRM in the original pixel space, we use them in the latent space of a pre-trained variational auto-encoder (VAE) following \citep{esser2021taming}. Specifically, the VAE encodes original (3,32,32) CIFAR10 images to have shapes (3,8,8) in the latent space, where these latent codes are inputs to WRM or FRM. Samples from the LFDs in the latent space are then decoded via the VAE before passing into the ResNet-18 image classifier.
    \item When using FRM, rather than training $K$ ResNet blocks (one per each $X|Y$), we train $K$ blocks that estimate the continuous-time flow \eqref{eq:neural_ode_T} via three fixed-stage RK4 steps.
\end{itemize}
Regarding other inputs to our proposed Algorithm \ref{alg:min_max}, the regularization parameter $\gamma=5$, total iteration $N=20000$ (binary CIFAR10) or $N=15000$ (full CIFAR10), and the number of inner loops $N_{\rm inner}=2$. The architecture of $f(x;\theta)$ in our flow model is a CNN auto-encoder with an encoder having channels $3-128-c_1\cdot 128-c_2\cdot128$, kernel sizes $3,3,3$, strides $1,2,1$, and padding $1,1,1$. The decoder has channels $c_2\cdot128-c_1\cdot128-128-3$, kernel sizes $3,4,3$, strides $1,2,1$, and padding $1,1,1$. We let $(c_1,c_2)=(2,4)$ on binary CIFAR10 and $(c_1,c_2)=(3,6)$ on full CIFAR10.

\paragraph{WDRO with Gaussian smoothed discrete LFD.}

We explain how the WDRO method works for this problem. For class $k=1,2$, let $x^i_k$ be the $i$-th training sample from classes $k$. Suppose there are $n_1$ training samples in class 1 and $n_2$ in class 2. Denote $\{x^i\}_{i=1}^{n}$ as the collection of $n=n_1+n_2$ samples from both classes. Then, given radius $\{\varepsilon_1, \varepsilon_2\}$ and training samples $\{x^i\}$, WDRO solves the following finite-dimensional convex program \citep[Lemma 2]{xie2021robust} to find the LFDs supported on $\{x^i\}$:
\begin{align}
\max_{\makecell[l]{p_1,p_2 \in \mathbb{R}_+^n, \\ \gamma_1,\gamma_2 \in \mathbb{R}_+^{n \times n}}} \quad & \sum_{l=1}^n \min \{p_1^l, p_2^l\} \label{eq:cvx_prob}  \\[-1.25em]
\text{subject to} \quad & \sum_{l=1}^n \sum_{m=1}^n \gamma_{k}^{lm}||x^l-x^m||_2 \leq \varepsilon_k, \quad k = 1, 2 \nonumber\\[-0.5em]
& \sum_{m=1}^n \gamma_{1}^{lm} = \frac{1}{n_1} \text{ and } \sum_{m=1}^n \gamma_{2}^{lm} = 0, \quad 1\leq l \leq n_1 \nonumber\\[-0.5em]
& \sum_{m=1}^n \gamma_{1}^{lm} = 0 \text{ and } \sum_{m=1}^n \gamma_{2}^{lm} = \frac{1}{n-n_1}, \quad n_1+1\leq l \leq n \nonumber\\[-0.5em]
& \sum_{l=1}^n \gamma_{k}^{lm} = p^m_k, \quad 1 \leq m \leq n, \quad k = 1, 2 \nonumber
\end{align}
Note that \eqref{eq:cvx_prob} has $O(n^2)$ decision variables, so the complexity of solving this linear program is on the order of $O(n^4)$. Thus, solving \eqref{eq:cvx_prob} is computationally infeasible for large sample sizes (e.g., when $n\sim 10^3$). 
Finally, to sample from the \textit{discrete} LFDs obtained from \eqref{eq:cvx_prob}, we follow \citep[Section 3.4]{xie2021robust} to use kernel smoothing with the Gaussian kernel under bandwidth $h$, so that the smoothed LFD for class $k$ becomes a Gaussian mixture with $n$ components, where the $i$-th component $\mathcal{N}(x^i, h^2I)$ is chosen with probability $p_k^i$.

\end{document}